\documentclass{article}

\PassOptionsToPackage{numbers, compress}{natbib}
\usepackage[final]{neurips_2023}

\usepackage[utf8]{inputenc} %
\usepackage[T1]{fontenc}    %
\usepackage{url}            %
\usepackage{booktabs}       %
\usepackage{amsfonts}       %
\usepackage{nicefrac}       %
\usepackage{microtype}      %
\usepackage{xcolor}         %
\usepackage{amsmath}
\usepackage{amsfonts}
\usepackage{amssymb}
\usepackage{graphicx}
\usepackage{multirow}
\usepackage{xspace}
\usepackage{stmaryrd}
\usepackage{braket}
\usepackage{floatrow}
\usepackage{titletoc}
\usepackage[frozencache=true,cachedir=.]{minted}
\usepackage{hhline}
\usepackage{hyperref}       %
\usepackage{mathtools}
\usepackage{amsthm}
\usepackage{boldline}
\usepackage{float}
\floatstyle{plaintop}
\restylefloat{table}

\newcommand{\genours}{\textsc{NFN}\xspace}
\newcommand{\ours}{$\textrm{NFN}_\textrm{NP}$\xspace}
\newcommand{\fullours}{$\textrm{NFN}_\textrm{HNP}$\xspace}
\newcommand{\ptours}{$\textrm{NFN}_\textrm{PT}$\xspace}
\newcommand{\statnet}{\textsc{StatNN}\xspace}
\newcommand{\mlp}{\textsc{MLP}\xspace}
\newcommand{\mlpaug}{$\textrm{MLP}_\textrm{Aug}$\xspace}

\newcommand{\group}{\tilde{\mathcal{S}}}
\newcommand{\iogroup}{\mathcal{S}}
\newcommand{\calU}{\mathcal{U}}
\newcommand{\calW}{\mathcal{W}}
\newcommand{\calV}{\mathcal{V}}
\newcommand{\UU}{\mathbb{U}}
\newcommand{\WW}{\mathbb{W}}
\newcommand{\VV}{\mathbb{V}}
\newcommand{\II}{\mathbb{I}}
\newcommand{\R}{\mathbb{R}}
\newcommand{\E}{\mathbb{E}}
\newcommand{\ci}{c_i}
\newcommand{\co}{c_o}

\newcommand{\paren}[1]{\left( #1 \right)}
\newcommand{\sqbra}[1]{\left[ #1 \right]}
\newcommand{\wt}[2]{W^{(#1)}_{#2}}
\newcommand{\wtmat}[1]{W^{(#1)}}
\newcommand{\bs}[2]{v^{(#1)}_{#2}}
\newcommand{\bsvec}[1]{v^{(#1)}}
\newcommand{\idx}[3]{#1^{#2}_{#3}}
\newcommand{\ve}{\text{vec}}
\newcommand{\aaa}[2]{a^{#1}_{#2}}
\newcommand{\bb}[2]{b^{#1}_{#2}}
\newcommand{\cc}[2]{c^{#1}_{#2}}
\newcommand{\dd}[2]{d^{#1}_{#2}}
\newcommand{\case}[1]{\begin{cases} #1 \end{cases}}
\newcommand{\hyperio}{H}
\newcommand{\hyper}{\tilde{H}}
\newcommand{\blue}[1]{\textcolor{blue}{#1}}

\newcommand{\invio}{P}
\newcommand{\inv}{\tilde{P}}
\newcommand{\range}[1]{\llbracket #1 \rrbracket}
\newcommand{\orb}{\text{Orbit}}
\newcommand{\inr}{\text{SIREN}}

\newfloatcommand{capbtabbox}{table}[][\FBwidth]

\newtheorem{proposition}{Proposition}

\title{Permutation Equivariant Neural Functionals}

\author{%
  Allan Zhou$^1$ \quad Kaien Yang$^1$ \quad Kaylee Burns$^1$ \quad Adriano Cardace$^2$ \quad Yiding Jiang$^3$ \\
  \textbf{Samuel Sokota}$^3$ \quad \textbf{J. Zico Kolter}$^3$ \quad \textbf{Chelsea Finn}$^1$\\
  $^1$Stanford University \quad $^2$University of Bologna \quad $^3$Carnegie Mellon University \\
  \texttt{ayz@cs.stanford.edu} \\
}

\begin{document}

\maketitle

\begin{abstract}
This work studies the design of neural networks that can process the weights or gradients of other neural networks, which we refer to as \textit{neural functional networks} (NFNs). Despite a wide range of potential applications, including learned optimization, processing implicit neural representations, network editing, and policy evaluation, there are few unifying principles for designing effective architectures that process the weights of other networks.
We approach the design of neural functionals through the lens of symmetry, in particular by focusing on the permutation symmetries that arise in the weights of deep feedforward networks because hidden layer neurons have no inherent order.
We introduce a framework for building \textit{permutation equivariant} neural functionals, whose architectures encode these symmetries as an inductive bias.
The key building blocks of this framework are \textit{NF-Layers} (neural functional layers) that we constrain to be permutation equivariant through an appropriate parameter sharing scheme.
In our experiments, we find that permutation equivariant neural functionals are effective on a diverse set of tasks that require processing the weights of MLPs and CNNs, such as predicting classifier generalization, producing ``winning ticket'' sparsity masks for initializations, and classifying or editing implicit neural representations (INRs).
In addition, we provide code for our models and experiments\footnote{\url{https://github.com/AllanYangZhou/nfn}}.
\end{abstract}

\section{Introduction}
\label{sec:intro}

As deep neural networks have become increasingly prevalent across various domains, there has been a growing interest in techniques for processing their weights and gradients as data. Example applications include learnable optimizers for neural network training~\citep{bengio2013optimization,runarsson2000evolution,andrychowicz2016learning,metz2022velo}, extracting information from implicit neural representations of data~\citep{stanley2007compositional,mildenhall2020nerf,sitzmann2020implicit}, corrective editing of network weights~\citep{sinitsin2020editable,de2021editing,mitchell2021fast}, policy evaluation \citep{harb2020policy}, and Bayesian inference given networks as evidence \citep{sokota2022a}.
We refer to functions of a neural network's weight space (such as weights, gradients, or sparsity masks) as \textit{neural functionals}; when these functions are themselves neural networks, we call them \textit{neural functional networks} (\genours{}s).

In this work, we design neural functional networks by incorporating relevant symmetries directly into the architecture, following a general line of work in ``geometric deep learning''~\citep{cohen2016group,ravanbakhsh2017equivariance,kondor2018generalization,bronstein2021geometric}.
For neural functionals, the symmetries of interest are transformations of a network's weights that preserve the network's behavior. 
In particular, we focus on \textit{neuron permutation symmetries}, which are those that arise from the fact that the neurons of hidden layers have no inherent order.

\begin{figure}
    \centering
    \includegraphics[width=\textwidth]{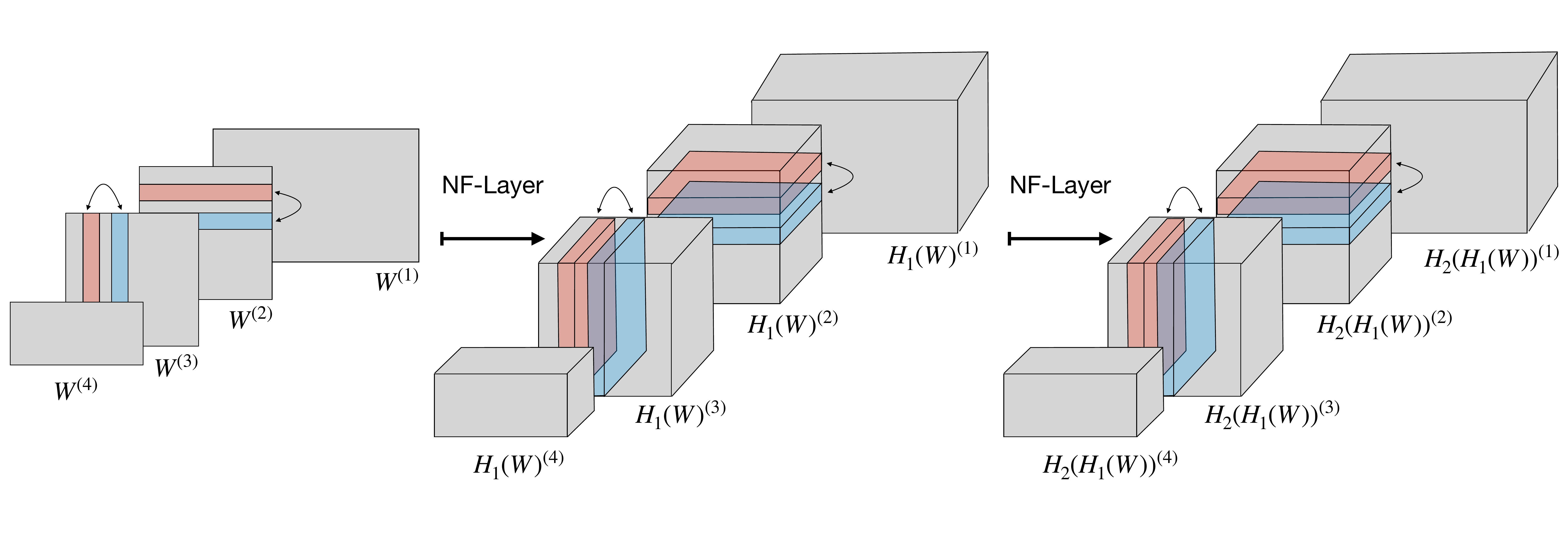}
    \caption{The internal operation of our permutation equivariant neural functionals (\genours{}s). The \genours{} processes the input weights through a series of equivariant NF-Layers, with each one producing \textit{weight-space features} with varying numbers of channels. In this example, a neuron permutation symmetry simultaneously permutes the rows of $\wtmat{2}$ and the columns of $\wtmat{3}$. This permutation propagates through the NFN in an equivariant manner.}
    \vspace{-2em}
    \label{fig:hnet-diagram}
\end{figure}

Neuron permutation symmetries are simplest in feedforward networks, such as multilayer perceptrons (MLPs) and basic convolutional neural networks (CNNs). These symmetries are induced by the fact that the neurons in each hidden layer of a feedforward network can be arbitrarily permuted without changing its behavior~\citep{hecht1990algebraic}. In MLPs, permuting the neurons in hidden layer $i$ corresponds to permuting the rows of the weight matrix $\wtmat{i}$, and the columns of the next weight matrix $\wtmat{i+1}$ as shown on the left-hand side of Figure~\ref{fig:hnet-diagram}. Note that the same permutation must be applied to the rows $\wtmat{i}$ and columns of $\wtmat{i+1}$, since applying \textit{different} permutations generally changes network behavior and hence does not constitute a neuron permutation symmetry.

We introduce a new framework for constructing neural functional networks that are invariant or equivariant to neuron permutation symmetries. Our framework extends a long line of work on permutation equivariant architectures~\citep{qi2016pointnet,zaheer2017deep,hartford2018deep,thiede2020general,maron2020learning} that design equivariant layers for a particular permutation symmetry of interest.
Specifically, we introduce neural functional layers (NF-Layers) that operate on weight-space features (see Figure~\ref{fig:hnet-diagram}) while being equivariant to neuron permutation symmetries.
Composing these NF-Layers with pointwise non-linearities produces equivariant neural functionals.

We propose different NF-Layers depending on the assumed symmetries of the input weight space: either only the hidden neurons of the feedforward network can be permuted (hidden neuron permutation, HNP), or all neurons, including inputs and outputs, can be permuted (neuron permutation, NP).
Although the HNP assumption is typically more appropriate, the corresponding NF-Layers can be parameter inefficient and computationally infeasible in some settings. In contrast, NF-Layers derived under NP assumptions often lead to much more efficient architectures, and, when combined with a positional encoding scheme we design, can even be effective on tasks that require breaking input and output symmetry.
For situations where invariance is required, we also define invariant NF-Layers that can be applied on top of equivariant weight-space features.

Finally, we investigate the applications of permutation equivariant neural functionals on tasks involving both feedforward MLPs and CNNs. Our first two tasks require (1) predicting the test accuracy of CNN image classifiers and (2) classifying implicit neural representations (INRs) of images and 3D shapes.
We then evaluate NFNs on their ability to (3) predict good sparsity masks for initializations (also called \textit{winning tickets}~\citep{frankle2018lottery}), and on (4) a weight-space ``style-editing'' task where the goal is to modify the content an INR encodes by directly editing its weights. In multiple experiments across these diverse settings, we find that permutation equivariant neural functionals consistently outperform non-equivariant methods and are effective for solving weight space tasks.

\textbf{Relation to DWSNets.} The recent work of Navon et al. [45] recognized the potential of leveraging weight space symmetries to build equivariant architectures on deep weight spaces; they characterize a weight-space layer which is mathematically equivalent to our NF-Layer in the HNP setting. Their work additionally studies interesting universality properties of the resulting equivariant architectures, and demonstrates strong empirical results for a suite of tasks that require processing the weights of MLPs.
Our framework additionally introduces the NP setting, where we make stronger symmetry assumptions to develop equivariant layers with improved parameter efficiency and practical scalability. We also extend our NFN variants to process convolutional neural networks (CNNs) as input, leading to applications such as predicting the generalization of CNN classifiers (Section 3.1).

\section{Equivariant neural functionals}
\label{sec:method}
\begingroup
\renewcommand{\arraystretch}{1.5}
\begin{table}
    \centering
    \caption{Permutation symmetries of $L$-layer feedforward networks with $n_0, \ldots, n_L$ neurons at each layer. All feedforward networks are invariant under hidden neuron permutations (HNP), while NP assumes that input and output neurons can also be permuted. We show the corresponding equivariant NF-Layers which process weight-space features from $\calU$, with $\ci$ input channels and $\co$ output channels.
    }
    \label{table:notation}
    \begin{tabular}{c|c|c|c|c}
        \multirow{2}{*}{\textbf{Group}} & \multirow{2}{*}{\textbf{Abbrv}} & \multirow{2}{*}{\textbf{Permutable layers}} & \multicolumn{2}{|c}{\textbf{Equivariant NF-Layer}} \\
        & & & Signature & Parameter count\\
        \hline
        $\iogroup=\prod_{i=0}^L S_{n_i}$ & NP & All layers & $\hyperio: \calU^{\ci} \rightarrow \calU^{\co}$ & $O(\ci\co L^2)$\\
        $\group = \prod_{i=1}^{L-1} S_{n_i}$ & HNP & Hidden layers & $\hyper:\calU^{\ci} \rightarrow \calU^{\co}$ & $O\paren{\ci\co(L+n_0+n_L)^2}$\\
        --- & --- & None & $T:\calU^{\ci} \rightarrow \calU^{\co}$ & $\ci\co\dim(\calU)^2$
    \end{tabular}
\end{table}
\endgroup
We begin by setting up basic concepts related to (hidden) neuron permutation symmetries, before defining the equivariant NF-Layers in Sec.~\ref{sec:NF-Layer} and invariant NF-Layers in Sec.~\ref{sec:invariant}.
\subsection{Preliminaries}
\label{sec:prelim}
Consider an $L$-layer feedforward network having $n_{i}$ neurons at layer $i$, with $n_0$ and $n_L$ being the input and output dimensions, respectively. The network is parameterized by weights $W=\Set{\wtmat{i}\in\R^{n_i \times n_{i-1}} \mid i \in \range{1..L}}$ and biases $v = \set{\bsvec{i} \in \R^{n_i} \mid i \in\range{1..L}}$. We denote the combined collection $U \coloneqq (W,v)$ belonging to weight space, $\calU\coloneqq\calW \times \calV$.

Since the neurons in a hidden layer $i \in \{1,\cdots, L-1\}$ have no inherent ordering, the network is invariant to the symmetric group $S_{n_i}$ of permutations of the neurons in layer $i$. This reasoning applies to every hidden layer, so the network is invariant to $\group \coloneqq S_{n_1}\times \cdots \times S_{n_{L-1}}$, which we refer to as the \textbf{hidden neuron permutation} (HNP) group.
Under the stronger assumption that the input and output neurons are also unordered, the network is invariant to $\iogroup\coloneqq S_0 \times \cdots \times S_{n_L}$, which we refer to as the \textbf{neuron permutation} (NP) group.
We focus on the NP setting throughout the main text, and treat the HNP case in Appendix~\ref{appendix:equiv}. See Table~\ref{table:notation} for a concise summary of the relevant notation for each symmetry group we consider.

Consider an MLP and a permutation $\sigma=(\sigma_0, \cdots, \sigma_{L}) \in \iogroup$. The action of the neuron permutation group is to permute the rows of each weight matrix $W^{(i)}$ by $\sigma_i$, and the columns by $\sigma_{i-1}$. Each bias vector $v^{(i)}$ is also permuted by $\sigma_i$. So the action is $\sigma U \coloneqq (\sigma W, \sigma v)$, where:
\begin{equation}
    \label{eq:action}
    \idx{\sqbra{\sigma W}}{i}{jk} = \wt{i}{\sigma_{i}^{-1}(j), \sigma_{i-1}^{-1}(k)}, \quad
    \idx{\sqbra{\sigma v}}{i}{j} = \bs{i}{\sigma_{i}^{-1}(j)}.
\end{equation}

Until now we have used $U=(W,v)$ to denote actual weights and biases, but the inputs to a neural functional layer could be any weight-space \textit{feature} such as a gradient, sparsity mask, or the output of a previous NF-Layer (Figure~\ref{fig:hnet-diagram}). Moreover, we may consider inputs with $c \geq 1$ feature channels, belonging to $\calU^c = \bigoplus_{i=1}^c \calU$, the direct sum of $c$ copies of $\calU$. Concretely, each $U\in\calU^c$ consists of weights $W=\Set{\wtmat{i}\in\R^{n_i\times n_{i-1} \times c} \mid i\in\range{1..L}}$ and biases $v=\Set{\bsvec{i}\in\R^{n_i \times c} \mid i \in \range{1..L}}$, with the channels in the final dimension. The action defined in Eq.~\ref{eq:action} extends to the multiple channel case if we define $\wt{i}{jk} := \wt{i}{j,k,:} \in \R^c$ and $\bs{i}{j} := \bs{i}{j,:} \in \R^c$.

The focus of this work is on making neural functionals that are equivariant (or invariant) to neuron permutation symmetries. Letting $\ci$ and $\co$ be the number of input and output channels, we refer to a function $f:\calU^{\ci} \rightarrow \calU^{\co}$ as $\iogroup$\textbf{-equivariant} if $\sigma f(U) = f(\sigma U)$ for all $\sigma \in \iogroup$ and $U\in \calU^{\ci},$
where the action of $\iogroup$ on the input and output spaces is defined by Eq.~\ref{eq:action}. Similarly, a function $f:\calU^c\rightarrow \mathbb{R}$ is $\iogroup$\textbf{-invariant} if $f(\sigma U) = f(U)$ for all $\sigma$ and $U$.

If $f,g$ are equivariant, then their composition $f \circ g$ is also equivariant; if $g$ is equivariant and $f$ is invariant, then $f\circ g$ is invariant. Since pointwise nonlinearities are already permutation equivariant, our remaining task is to design a \textit{linear} NF-Layer that is $\iogroup$-equivariant. We can then construct equivariant neural functionals by stacking these NF-Layers with pointwise nonlinearities.

\subsection{Equivariant NF-Layers}
\label{sec:NF-Layer}
\begin{figure}
    \centering
    \includegraphics[width=\textwidth]{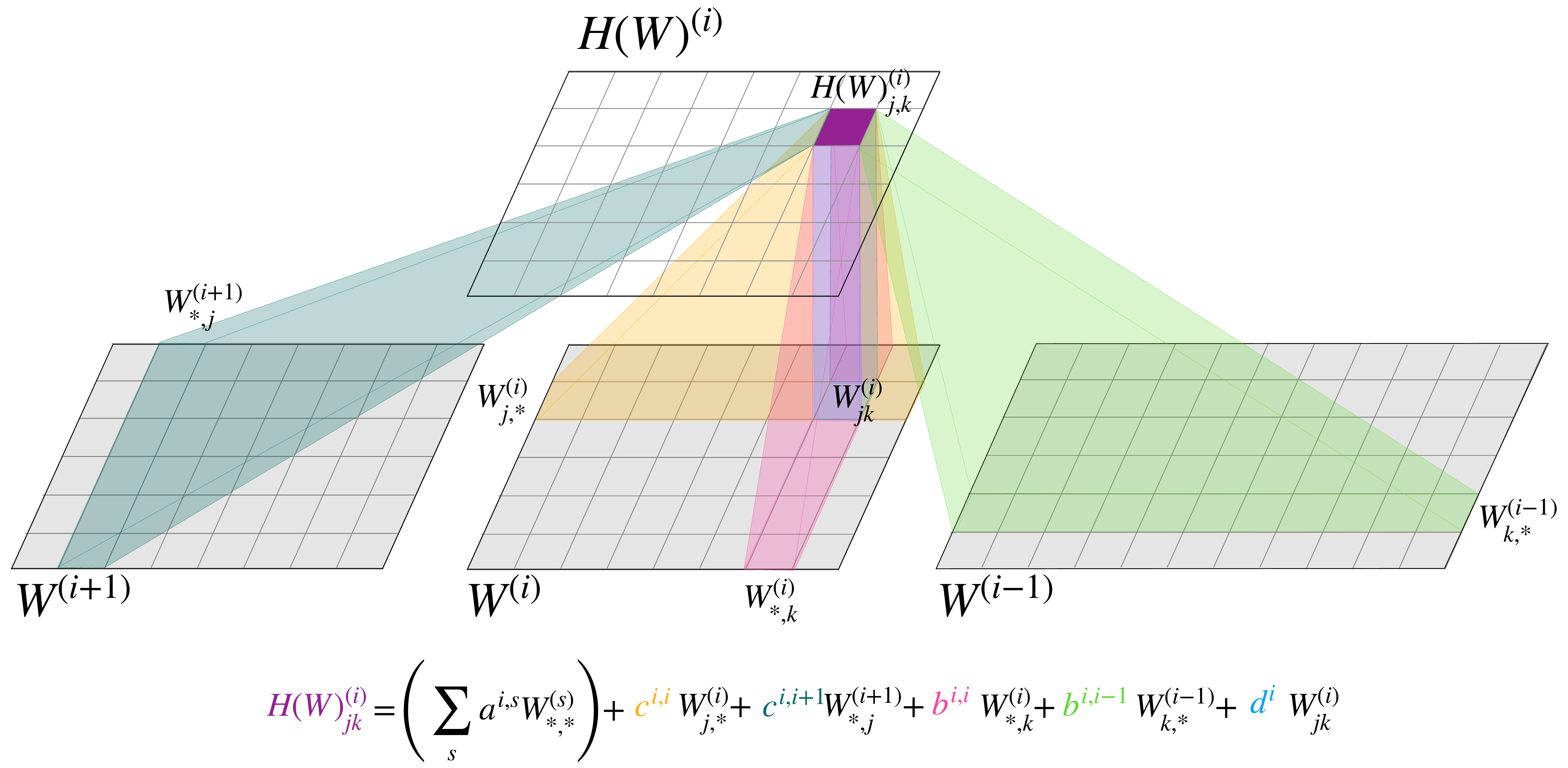}
    \caption{A permutation equivariant NF-Layer takes in weight-space features as input (bottom) and outputs transformed features (top), while respecting the neuron permutation symmetries of feedforward networks. This illustrates the computation of a single output element $\idx{H(W)}{i}{jk}$, defined in Eq.~\ref{eq:layer-simple}. Each output is a weighted combination of rows or column sums of the input weights, which preserves permutation symmetry. The first term contributes a weighted combination of row-and-column sums from \textit{every} input weight, though this is omitted for visual clarity.}
    \label{fig:layer-operation}
\end{figure}

We now construct a linear $\iogroup$-equivariant layer that serves as a key building block for neural functional networks. In the single channel case, we begin with generic linear layers $T(\cdot;\theta): \ve(U) \mapsto \theta \ve(U)$, where $\ve(U)\in\R^{\dim(U)}$ is $U$ flattened as a vector and $\theta \in \R^{\dim(\calU) \times \dim(\calU)}$ is a matrix of parameters. We show in Appendix~\ref{sec:orbit-motivation} that \textbf{any} $\iogroup$-equivariant $T(\cdot;\theta)$ must satisfy a system of constraints on $\theta$ known as equivariant \textit{parameter sharing}. We derive this parameter sharing by partitioning the entries of $\theta$ by the orbits of their indices under the action of $\iogroup$, with parameters shared in each orbit~\citep{ravanbakhsh2017equivariance}. Table~\ref{table:theta-constraint} of the appendix describes the parameter sharing in detail. 

Equivariant parameter sharing reduces the matrix-vector product $\theta \ve(U)$ to the NF-Layer we now present. For simplicity we ignore $\calV$ and assume here that $\calU=\calW$ and defer the full form to Eq.~\ref{eq:layer-full} in the appendix. Then $\hyperio:\calW^{\ci} \rightarrow \calW^{\co}$ maps input $\paren{\wtmat{1}, \cdots, \wtmat{L}}$ to $\left(\hyperio(W)^{(1)},\cdots, \hyperio(W)^{(L)}\right)$. Recall that the inputs are not necessarily weights, but could be arbitrary weight-space features including the output of a previous NF-Layer. For $\wtmat{i}\in\R^{n_i \times n_{i-1} \times \ci}$, the corresponding output is $\hyperio(W)^{(i)} \in \R^{n_i \times n_{i-1} \times \co}$ with entries computed:
\begin{equation}
    \label{eq:layer-simple}
    \hyperio(W)^{(i)}_{jk} = \left(\sum_s \blue{a^{i,s}} \wt{s}{\star,\star}\right)
    + \blue{b^{i,i}} \wt{i}{\star,k} + \blue{b^{i,i-1}}\wt{i-1}{k,\star}
    + \blue{c^{i,i}} \wt{i}{j,\star}
    + \blue{c^{i,i+1}} \wt{i+1}{\star,j}
    + \blue{d^i} \wt{i}{jk}.
\end{equation}
Note that the terms involving $\wtmat{i-1}$ or $\wtmat{i+1}$ should be omitted for $i=0$ and $i=L$, respectively, and $\star$ denotes summation or averaging over either the rows or columns. Recall that in the multi-channel case, each $\wt{i}{jk}$ is a vector in $\R^{\ci}$ so each \blue{parameter} is a $\co \times \ci$ matrix. We also provide a concrete pseudocode description of $\hyperio$ in Appendix~\ref{appendix:pseudocode}. Figure~\ref{fig:layer-operation} visually illustrates the NF-Layer in the single-channel case, showing how the row or column sums from each input contribute to each output. To gain intuition for the operation of $\hyperio$, it is straightforward to check $\iogroup$-equivariance:

\begin{proposition}
The NF-Layer $\hyperio:\calU^{\ci} \rightarrow \calU^{\co}$ (Eq.~\ref{eq:layer-simple} and Eq.~\ref{eq:layer-full}) is $\iogroup$-equivariant, where the group's action on input and output spaces is defined by Eq.~\ref{eq:action}. Moreover, any linear $\iogroup$-equivariant map $T:\calU^{\ci}\rightarrow \calU^{\co}$ is equivalent to $H$ for some choice of parameters $a,b,c,d$.
\end{proposition}

\begin{proof}[Proof (sketch)] We can verify that $H$ satisfies the equivariance condition $\sqbra{\sigma H(W)}^{(i)}_{jk} = H(\sigma W)^{(i)}_{jk}$ for any $i,j,k$ by expanding each side of the equation using the definitions of the layer and action (Eq.~\ref{eq:action}). Moreover, Appendix~\ref{sec:orbit-motivation} shows that any $\iogroup$-equivariant linear map $T(\cdot,\theta)$ must have the same equivariant parameter sharing as $H$, meaning that it must be equivalent to $H$ for some choice of parameter $a,b,c,d$. See Appendix~\ref{appendix:equiv} for the full proof.
\end{proof}

Informally, the above proposition tells us that $H$ can express any linear $\iogroup$-equivariant function of a weight space. Since $\group$ is a subgroup of $\iogroup$, $\hyperio$ is also $\group$-equivariant. However, it does not express every possible linear $\group$-equivariant function. We derive the full $\group$-equivariant NF-Layer $\hyper:\calU \rightarrow \calU$ in Appendix~\ref{appendix:nonio-theory}. %

 Table~\ref{table:notation} summarizes the number of parameters (after parameter sharing) under different symmetry assumptions. While in general a linear layer $T(\cdot; \theta):\calU^{\ci}\rightarrow \calU^{\co}$ has $\ci\co\dim(\calU)^2$ parameters, the equivariant NF-Layers have significantly fewer free parameters due to parameter sharing. The $\iogroup$-equivariant layer $\hyperio$ has $O\paren{\ci\co L^2}$, while the $\group$-equivariant layer $\hyper$ has $O\paren{\ci\co(L + n_0 + n_L)^2}$ parameters. The latter's quadratic dependence on input and output dimensions can be prohibitive in some settings, such as in classification where the number of outputs can be tens of thousands.

\textbf{Extension to convolutional weight spaces.} 
In convolution layers, since neurons correspond to spatial \textit{channels}, we let $n_i$ denote the number of channels at the $i^\mathrm{th}$ layer. Each bias $\bsvec{i}\in\R^{n_i}$ has the same dimensions as in the fully connected case, so only the convolution filter needs to be treated differently since it has additional spatial dimension(s) that cannot be permuted. For example, consider a 1D CNN with filters $W = \Set{\wtmat{i}\in\R^{n_i \times n_{i - 1} \times w} | i\in\range{1..L}}$, where $n_i \times n_{i-1}$ are the output and input channel dimensions and $w$ is the filter width. We let $\wt{i}{jk} \coloneqq \wtmat{i}_{j,k,:} \in \R^{w}$ denote the $k^\mathrm{th}$ filter in the $j^\mathrm{th}$ output channel, then define the $\iogroup$-action the same way as in Eq.~\ref{eq:action}.

We immediately observe the similarities to multi-channel features: both add dimensions that are not permuted by the group action. In fact, suppose we have $c$-channel features $U \in \calU^c$ where $\calU$ is the weight-space of a 1D CNN. Then we combine the filter and channel dimensions of the weights, with $\wtmat{i} \in \R^{n_i \times n_{i-1} \times (cw)}$. This allows us to use the multi-channel NF-Layer $\hyperio:\calU^{w\ci}\rightarrow\calU^{w\co}$. Any further channel dimensions, such as those for 2D convolutions, can also be folded into the channel dimension.

It is common for CNNs in image classification to follow convolutional layers with pooling and fully connected (FC) layers, which opens the question of defining the $\iogroup$-action when layer $\ell$ is FC and layer $\ell-1$ is convolutional. If global spatial pooling removes all spatial dimensions from the output of $\ell -1$ (as in e.g., ResNets~\citep{he2015deep} and the Small CNN Zoo~\citep{unterthiner2020predicting}), then we can verify that the existing action definitions work without modification. We leave more complicated situations (e.g., when nontrivial spatial dimensions are flattened as input to FC layers) to future work.

\textbf{IO-encoding.} The $\iogroup$-equivariant layer $\hyperio$ is more parameter efficient than $\hyper$ (Table~\ref{table:notation}), but its NP assumptions are typically too strong. To resolve this problem, we can add either learned or fixed (sinusoidal) position embeddings to the columns of $\wtmat{1}$ and the rows of $\wtmat{L}$ and $\bsvec{L}$; this breaks the symmetry at input and output neurons even when using $\iogroup$-equivariant layers. In our experiments, we find that IO-encoding makes $\hyperio$ competitive or superior to $\hyper$, while using a fraction of the parameters.

\subsection{Invariant NF-Layers}
\label{sec:invariant}
Invariant neural functionals can be designed by composing multiple equivariant NF-Layers with an invariant NF-Layer, which can then be followed by an MLP. We define an $\iogroup$-invariant layer $\invio: \calU \rightarrow \R^{2L}$ by simply summing or averaging the weight matrices and bias vectors across any axis that has permutation symmetry, i.e., 
    $\invio(U) = \paren{\wt{1}{\star,\star},\cdots, \wt{L}{\star,\star}, \bs{1}{\star}, \cdots, \bs{L}{\star}}.$
We define the analogous $\group$-invariant layer $\inv$ in Eq.~\ref{eq:pooling-nonio} of the appendix.

\section{Experiments}
Our experiments evaluate permutation equivariant neural functionals on a variety of tasks that require either invariance (predicting CNN generalization and extracting information from INRs) or equivariance (predicting ``winning ticket'' sparsity masks and weight-space editing of INR content).

Throughout the experiments, we construct neural functional networks (\genours{}s) using the NF-Layers described in the previous section. Although the specific design varies depending on the task, we will broadly refer to our permutation equivariant NFNs as \ours{} and \fullours{}, depending on which NF-Layer variant they use (see Table~\ref{table:notation}). We also evaluate a ``pointwise'' ablation of our equivariant NF-Layer that ignores interactions between weights by only using the last term of Eq.~\ref{eq:layer-simple}, computing $\idx{H(W)}{i}{jk}\coloneqq\blue{d^i} \wt{i}{jk}$. We refer to NFNs that use this pointwise NF-Layer as \ptours.

Where feasible we also compare against neural functionals with standard FC layers, instead of equivariant NF-Layers. We optionally augment the training data with permutations (using Eq.~\ref{eq:action}) to encourage permutation symmetry. We refer to these methods as \mlp and \mlpaug.

\subsection{Predicting CNN generalization from weights}
\label{sec:pred_cnn_gen}
Why deep neural networks generalize despite being heavily overparameterized is a longstanding research problem in deep learning. One recent line of work has investigated the possibility of directly predicting the test accuracy of the models from the weights~\citep{unterthiner2020predicting,eilertsen2020classifying}. The goal is to study generalization in a data-driven fashion and ultimately identify useful patterns from the weights.

Prior methods develop various strategies for extracting potentially useful features from the weights before using them to predict the test accuracy~\citep{jiang2018predicting, yak2019towards, unterthiner2020predicting,jiang2021methods, martin2021implicit}. %
However, using hand-crafted features could fail to capture intricate correlations between the weights and test accuracy.
Instead, we explore using neural functionals to predict test accuracy from the \textit{raw weights} of feedforward convolutional neural networks (CNN) from the \textit{Small CNN Zoo} dataset~\citep{unterthiner2020predicting}, which contains thousands of CNN weights trained on several datasets with varied hyperparameters. We compare the predictive power of \fullours and \ours against a method of \citet{unterthiner2020predicting} that trains predictors on statistical features extracted from each weight and bias, and refer to it as \statnet. To measure the predictive performance of each method, we use \textit{Kendall's $\tau$}~\cite{kendall1938new}, a popular rank correlation metric with values in $[-1,1]$.

In Table~\ref{tab:pred_gen_zoo}, we show the results on two challenging subsets of Small CNN Zoo corresponding to CNNs trained on CIFAR-10-GS and SVHN-GS (GS stands for grayscaled). We see that \fullours consistently performs the best on both datasets by a significant margin, showing that having access to the full weights can increase predictive power over hand-designed features as in \statnet. Because the input and output dimensionalities are small on these datasets, \fullours only uses moderately more ($\sim 1.4\times$) parameters than \ours with equivalent depth and channel dimensions, while having significantly better performance.

\begin{table}[]
    \centering
    \caption{Test $\tau$ of generalization prediction methods on the Small CNN Zoo~\citep{unterthiner2020predicting}, which contains the weights and test accuracies of many small CNNs trained on different datasets, such as CIFAR-10-GS or SVHN-GS. \fullours outperforms other methods on both datasets. Uncertainties indicate max and min over two runs.}
    \label{tab:pred_gen_zoo}
    \begin{tabular}{crrrr}
      \toprule
       & \fullours & \ours & \statnet  \\
      \midrule
      CIFAR-10-GS & $\mathbf{0.934 \pm 0.001}$ & $0.922 \pm 0.001$ & $0.915 \pm 0.002$   \\
      SVHN-GS & $\mathbf{0.931 \pm 0.005}$  & $0.856 \pm 0.001$ & $0.843 \pm 0.000$    \\
      \bottomrule
    \end{tabular}
\end{table}

\subsection{Classifying implicit neural representations of images and 3D shapes}
\begin{table}
    \centering
    \caption{Classification train and test accuracies (\%) for implicit neural representations of MNIST, FashionMNIST, and CIFAR-10. Our equivariant NFNs outperform the MLP baselines, even when the MLP has permutation augmentations to encourage invariance. Uncertainties indicate standard error over three runs.}
        \begin{tabular}{crrrr}
      \toprule
       & \fullours & \ours & 
\mlp & \mlpaug \\
      \midrule
        CIFAR-10 & $44.1 \pm 0.471$ & $\mathbf{46.6 \pm 0.072}$ & $16.9 \pm 0.250$ & $18.9 \pm 0.432$ \\
        \midrule
        MNIST-10 & $92.5 \pm 0.071$ & $\mathbf{92.9 \pm 0.218}$ & $14.5 \pm 0.035$ & $21.0 \pm 0.172$ \\
        \midrule
        FashionMNIST & $72.7\pm1.53$ & $\mathbf{75.6\pm1.07}$ & $12.5\pm0.111$ & $15.9\pm0.181$ \\
      \bottomrule
    \end{tabular}
    \label{tab:inr-classification}
\end{table}

\begin{table}
    \centering
    \caption{Classification test accuracies (\%) for datasets of implicit neural representations (INRs) of either ShapeNet-10~\citep{shapenet2015} or ScanNet-10~\citep{dai2017scannet} Our equivariant NFNs outperform the MLP baselines and recent non-equivariant methods such as inr2vec \citep{2023inr2vec}. Uncertainties indicate standard error over three runs.}
    \scalebox{0.92}{
        \begin{tabular}{crrrrr}
      \toprule
       & \fullours & \ours & 
\mlp & \mlpaug & inr2vec\citep{2023inr2vec} \\
      \midrule
        ShapeNet-10 & $86.9\pm0.860$ & $\mathbf{88.7\pm0.461}$ & $25.4\pm0.121$ & $33.8\pm0.126$ & $39.1\pm 0.385$\\
        \midrule
        ScanNet-10 & $64.1\pm0.572$ & $\mathbf{65.9\pm1.10}$ & $ 32.9\pm0.351$ & $45.5\pm0.126 $ & $38.2\pm0.409$ \\
      \bottomrule
    \end{tabular}}
    \label{tab:inr_3d-classification}
\end{table}

\label{sec:classifying}
Given the rise of implicit neural representations (INRs) that encode data such as images and 3D-scenes~\citep{stanley2007compositional,mescheder2019occupancy,chen2019learning,park2019deepsdf,sitzmann2020implicit,mildenhall2020nerf,dupont2021generative,dupont2022data}, it is natural to wonder how to extract information about the original data directly from the weights.

In this task, our goal is to classify the contents of INRs given only the weights as input. We consider datasets of SIRENs~\citep{sitzmann2020implicit} that encode images (MNIST~\citep{lecun2010mnist}, FashionMNIST~\cite{xiao2017/online}, and CIFAR~\citep{krizhevsky2009learning}) and 3D shapes (ShapeNet-10 and ScanNet-10~\citep{qin2019pointdan}). For image datasets each SIREN network represents the mapping from pixel coordinate to RGB (or grayscale) value for a single image, while for 3D shapes each network is a signed (or unsigned) distance function encoding a single shape.
Each dataset of SIREN weights is split into training, validation, and testing sets.

We construct and train invariant neural functionals to classify the INRs, and compare their performance against the \mlp and \mlpaug baselines, which are three-layer MLPs with ReLU activations and 1,000 hidden units per layer. For the 3D-shape datasets we also report the performance of inr2vec~\citep{2023inr2vec}, a recent non-equivariant method with results on classifying 3D shapes from INR weights. Note that inr2vec's original setting assumes that all INRs in a dataset are trained from the same shared initialization, whereas our problem setting makes no such assumption and allows INRs to be trained from random and independent initializations.

The results in Table~\ref{tab:inr-classification} and Table~\ref{tab:inr_3d-classification} show that \fullours and \ours consistently achieve higher test accuracies than the baseline methods on both datasets. In addition to superior generalization, Tables~\ref{tab:inr-classification-full}-\ref{tab:inr_3d-classification-full} in the appendix show that NFNs are also usually better at fitting the training data (higher train accuracy). The MLPs struggle to even fit the training data, \textit{especially} under permutations augmentations, even with the same number of parameters as the NFNs.
Interestingly, \ours matches or exceeds \fullours performance on both CIFAR-10 and the 3D-shape datasets while using fewer parameters (e.g., $35\%$ as many parameters on CIFAR-10).

\subsection{Predicting ``winning ticket'' masks from initialization}

\begin{table}
    \centering
    \caption{Test accuracy (\%) of training with winning tickets (95\% sparsity masks) produced either by running IMP or predicted by an \genours{}. We also show the performance of Random ticket (random mask of equivalent sparsity level), and Dense training (no sparsity). We show results for MLPs (trained on MNIST) and CNNs (trained on CIFAR-10). Uncertainties show standard error over initializations.}
    \begin{tabular}{crr|rrr}
      \toprule
       & Dense & IMP & Random & \ours & \ptours  \\
      \midrule
      CIFAR-10 & $63.1\pm 0.06$ & $44.0\pm 0.06$ & $21.1\pm 0.26$ & $\mathbf{41.4\pm 0.08}$ & $\mathbf{42.6\pm 0.07}$  \\
      MNIST & $97.8\pm 0.0$ &  $96.2\pm 0.04$ & $89.6\pm 0.36$  & $\mathbf{94.8\pm 0.01}$ & $\mathbf{95.0\pm 0.01}$ \\
      \bottomrule
    \end{tabular}
    \label{tab:lth}
\end{table}

The Lottery Ticket Hypothesis~\cite[LTH]{frankle2018lottery,frankle2019stabilizing} conjectures the existence of \textit{winning tickets}, or sparse initializations that train to the same final performance as dense networks, and showed their existence in some settings through iterative magnitude pruning (IMP). IMP retroactively finds a winning ticket by pruning \textit{trained} models by magnitude; however, finding the winning ticket from only the initialization without training remains challenging.

We demonstrate that permutation equivariant neural functionals are a promising approach for finding winning tickets at initialization by learning over datasets of initializations and their winning tickets.
Let $U_0 \in \calU$ be an initialization and let the sparsity mask $M \in  \{0,1\}^{\text{dim}(\calU)}$ be a winning ticket for the initialization, with zeros indicating that the corresponding entries of $U_0$ should be pruned.
The goal is to predict a winning ticket $\hat{M}$ given a held out initialization $U_0$, such that the MLP initialized with $U_0$ and sparsity pattern $\hat{M}$ will achieve a high test accuracy after training.

We construct a conditional variational autoencoder~\citep[cVAE]{kingma2013auto,sohn2015learning} that learns a generative model of the winning tickets conditioned on initialization and train on datasets of (initialization, ticket) pairs found by one step of IMP with a sparsity level of $P_m=0.95$ for both MLPs trained on MNIST and CNNs trained on CIFAR-10. Table~\ref{tab:lth} compares the performance of tickets predicted by equivariant neural functionals against IMP tickets and random tickets. We generate random tickets by randomly sampling sparsity mask entries from $\text{Bernoulli}(1-P_m)$. In this setting, \fullours is prohibitively parameter inefficient, but \ours is able to recover test accuracies that are close to that of IMP pruned networks in CIFAR-10 and MNIST, respectively. Somewhat surprisingly, \ptours performs just as well as the other \genours{}s, indicating that one can approach IMP performance in these settings without considering interactions between weights or layers. Appendix~\ref{appendix:lth-analysis} further analyzes how \ptours learns to prune.

\subsection{Weight space style editing}
\begin{figure}
\begin{floatrow}
\ffigbox{%
  \includegraphics[width=0.40\textwidth]{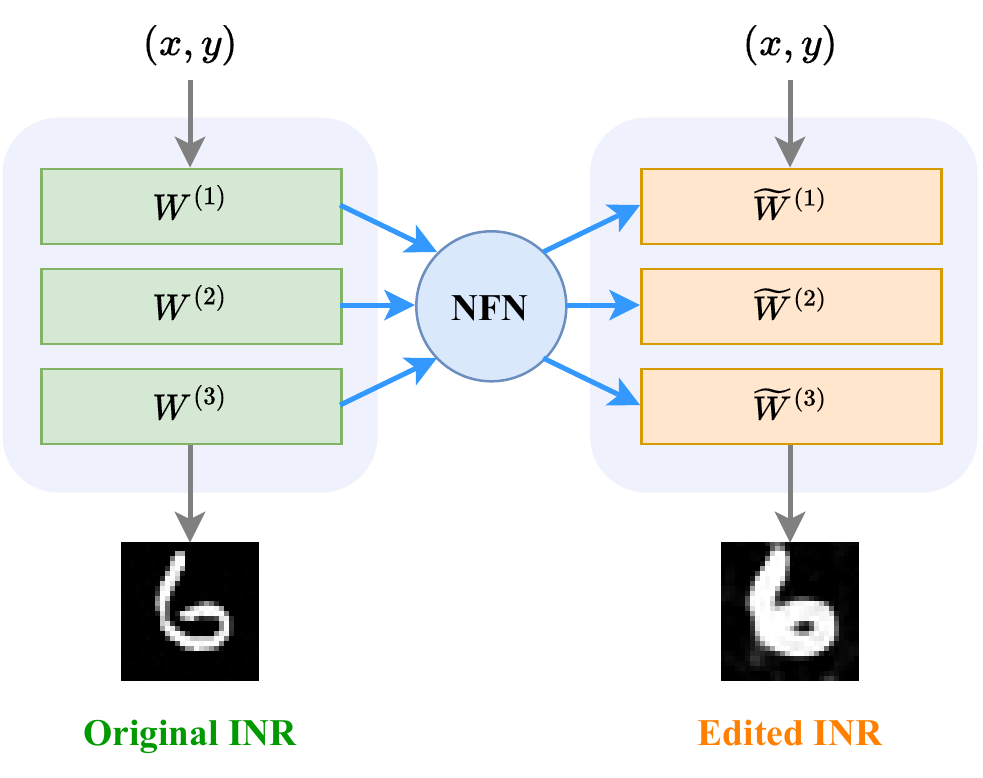}
}{
  \caption{In weight-space style editing, an \genours{} directly edits the weights of an INR to alter the content it encodes. In this example, the \genours{} edits the weights to dilate the encoded image.}
  \label{fig:style-editing}
}
\capbtabbox{%
  \begin{tabular}{lccc} \toprule
  Method & Contrast & Dilate \\
  & (CIFAR-10) & (MNIST) \\ \midrule
  \mlp & $0.031$ & $0.306$ \\
  \mlpaug & $0.029$ & $0.307$ \\
  \ptours & $0.029$ & $0.197$ \\
  \fullours & $\mathbf{0.021}$ & $\mathbf{0.070}$ \\
  \ours & $\mathbf{0.020}$ & $\mathbf{0.068}$ \\\bottomrule
  \end{tabular}
}{%
  \caption{Test mean squared error (lower is better) between weight-space editing methods and ground-truth image-space transformations.}%
  \label{tab:siren-editing-results}
}
\end{floatrow}
\end{figure}

\begin{figure}
    \centering
    \includegraphics[width=0.9\textwidth]{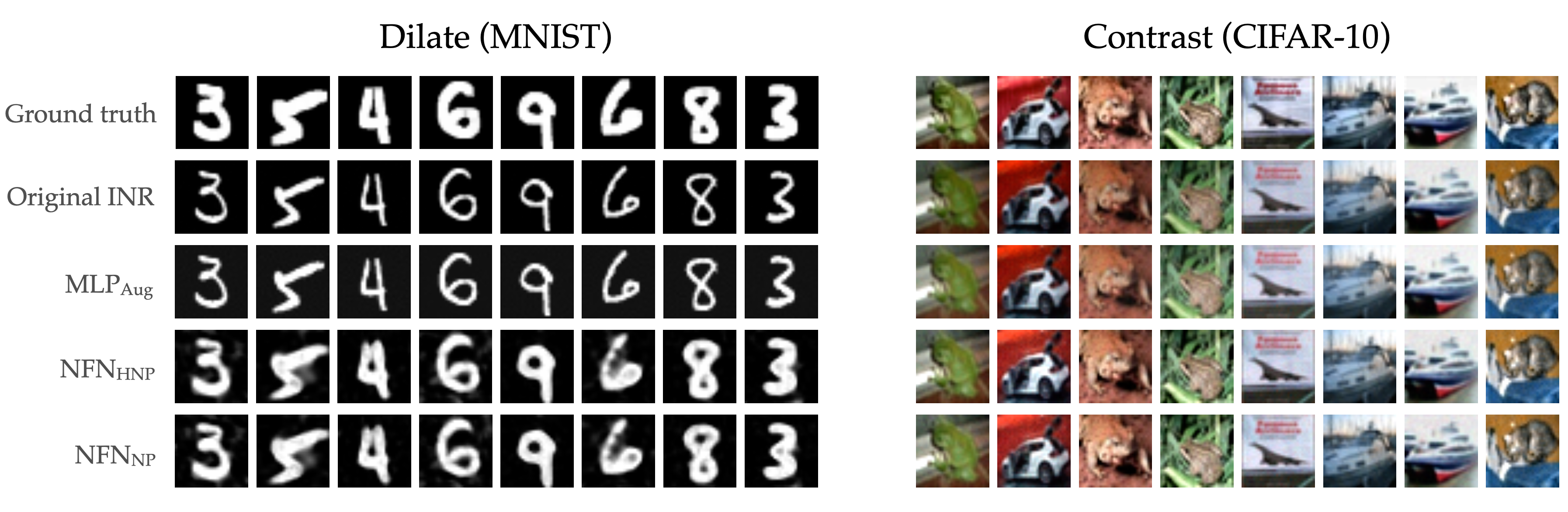}
    \caption{Random qualitative samples of INR editing behavior on the Dilate (MNIST) and Contrast (CIFAR-10) editing tasks. The first row shows the image produced by the original INR, while the rows below show the result of editing the INR weights with an \genours{}. The difference between \mlp neural functionals and equivariant neural functionals is especially pronounced on the more challenging Dilate tasks, which require modifying the geometry of the image. In the Contrast tasks, the \mlp baseline produces dimmer images compared to the ground truth, which is especially evident in the second and third columns.}
    \label{fig:style-editing-samples}
\end{figure}

Another potentially useful application of neural functionals is to edit (i.e., transform) the weights of a given INR to alter the content that it encodes. In particular, the goal of this task is to edit the weights of a trained SIREN to alter its encoded image (Figure~\ref{fig:style-editing}). We evaluate two editing tasks: (1) making MNIST digits thicker via image dilation (\textbf{Dilate}), and (2) increasing image contrast on CIFAR-10 (\textbf{Contrast}). Both of these tasks require neural functionals to process \textit{the relationships between different pixels} to successfully solve the task.

To produce training data for this task, we use standard image processing libraries~\citep[OpenCV]{itseez2015opencv} to dilate or increase the contrast of the MNIST and CIFAR-10 images, respectively. The training objective is to minimize the mean squared error between the image generated by the \genours{}-edited INR and the image produced by image processing.
We construct equivariant neural functionals to edit the INR weights, and compare them against MLP-based neural functionals with and without permutation augmentation.

Table~\ref{tab:siren-editing-results} shows that permutation equivariant neural functionals (\fullours and \ours) achieve significantly better test MSE when editing held out INRs compared to other methods, on both the Dilate (MNIST) and Contrast (CIFAR-10) tasks. In other words, they produce results that are closest to the ``ground truth'' image-space processing operations for each task. The pointwise ablation \ptours performs significantly worse, indicating that accounting for interactions between weights and layers is important to accomplishing these tasks. Figure~\ref{fig:style-editing-samples} shows random qualitative samples of editing by different methods below the original (pre-edit) INR. We observe that \genours{}s are more effective than \mlpaug at dilating MNIST digits and increasing the contrast in CIFAR-10 images.

\section{Related work}
The permutation symmetries of neurons have been a topic of interest in the context of loss landscapes and model merging~\citep{garipov2018loss, brea2019weight,tatro2020optimizing,entezari2021role,ainsworth2022git}.
Other works have analyzed the degree of learned permutation symmetry in networks that process weights~\citep{unterthiner2020predicting} and studied ways of accounting for symmetries when measuring or encouraging diversity in the weight space~\citep{deutsch2019generative}.
However, these symmetries have not been a key consideration in architecture design for processing weight space objects
~\citep{andrychowicz2016learning, li2016learning, ha2016hypernetworks,krueger2017bayesian,zhang2018graph, deutsch2019generative, knyazev2021parameter}.
Instead, existing approaches try to encourage permutation equivariance through data augmentation~\citep{peebles2022learning, metz2022velo}. In contrast, this work directly encodes the equivariance of the weight space into our architecture design, which can result in much higher data and computational efficiency, as evidenced by the success of convolutional neural networks~\citep{lecun1995convolutional}.

Our work follows a long line of literature that incorporates structure and symmetry into neural network architectures~\citep{lecun1995convolutional, cohen2016group,ravanbakhsh2017equivariance,kondor2018generalization,cohen2018spherical,finzi2021practical}, including works that design equivariant layers for various permutation symmetries~\citep{qi2016pointnet,zaheer2017deep,hartford2018deep,thiede2020general,maron2020learning}. Our key contribution is applying the framework of \citet{ravanbakhsh2017equivariance} to the particular neuron permutation symmetries found in the weights of deep neural networks~\citep{hecht1990algebraic}, leading to the characterization of our equivariant NF-Layers. As discussed in Section~\ref{sec:intro}, \citet{navon2023equivariant} recently developed an equivariant weight-space layer that is equivalent to our NF-Layer in the HNP setting. Our work introduces the NP setting to improve parameter efficiency and scalability over the HNP setting, and extends beyond the fully connected case to handle convolutional weight space inputs.

\section{Conclusion}
This paper proposes a novel symmetry-inspired framework for the design of neural functional networks (NFNs), which process weight-space features such as weights, gradients, and sparsity masks. Our framework focuses on the permutation symmetries that arise in weight spaces due to the particular structure of neural networks. We introduce two equivariant NF-Layers as building blocks for NFNs, which differ in their underlying symmetry assumptions and parameter efficiency, then use them to construct a variety of permutation equivariant neural functionals. Experimental results across diverse settings demonstrate that permutation equivariant neural functionals outperform prior methods and are effective for solving weight-space tasks.

\textbf{Limitations and future work.} Although we believe this framework is a step toward the principled design of effective neural functionals, there remain multiple directions for improvement. 
One such direction would involve reducing the activation sizes produced by NF-Layers, which could be useful to scaling neural functionals to process the weights of very large networks. 
Another such direction would concern extending the NF-Layers to process weight inputs of more complex architectures such as ResNet~\citep{he2015deep} and Transformer~\citep{vaswani2017attention} weights, which would enable larger-scale applications.

\bibliographystyle{abbrvnat}
\bibliography{references}

\clearpage

\appendix

\startcontents
\printcontents{}{-1}{\section*{Appendix}}

\section{Equivariant NF-Layer pseudocode}
\label{appendix:pseudocode}
Here we present a multi-channel implementation of the $\iogroup$-equivariant NF-Layer presented in Eq.~\ref{eq:layer-simple} (which ignores biases), using PyTorch~\citep{paszke2019pytorch} and Einops-like~\citep{rogozhnikov2022einops} pseudocode. That is, it implements a linear layer $\hyperio:\calW^{\ci}\rightarrow\calW^{\co}$, where $\ci$ and $\co$ are the number of input and output channels.

Note that our actual implementation differs from this pseudocode in a few ways: (1) it supports the full weight space $\calU$ which includes biases, (2) it supports convolution weights as well as fully connected weights, and (3) it initializes parameters based on the fan-in of the NF-Layer, instead of from $\mathcal{N}(0,1)$.

\renewcommand{\theFancyVerbLine}{
  \sffamily\textcolor[rgb]{0.5,0.5,0.5}{\scriptsize\arabic{FancyVerbLine}}}
\begin{minted}[mathescape,
               linenos,
               numbersep=5pt,
               gobble=2,
               frame=lines,
               framesep=2mm]{python}
  class NPLayer(nn.Module):
    def __init__(self, L, co, ci):
        super().__init__()
        # initialize weights. co=output channels, ci=input channels.
        self.A = nn.Parameter(torch.randn(L, L, co, ci))
        self.B = nn.Parameter(torch.randn(L, co, ci))
        self.B_prev = nn.Parameter(torch.randn(L, co, ci))
        self.C = nn.Parameter(torch.randn(L, co, ci))
        self.C_next = nn.Parameter(torch.randn(L, co, ci))
        self.D = nn.Parameter(torch.randn(L, co, ci))

    def forward(self, W):
        # Input W is a list of $L$ weight-space tensors with shapes:
        #     [$(B, \ci, n_1, n_0), \ldots, ((B, \ci, n_{L}, n_{L-1}))$]
        # We return a list of $L$ tensors with shapes:
        #     [$(B, \co, n_1, n_0), \ldots, ((B, \co, n_{L}, n_{L-1}))$]
        # where $\ci,\co$ are the input and output channels.
        
        # Compute $\sqbra{\wtmat{1}_{\star,:},\cdots,\wtmat{L}_{\star,:}}$. Each has shape $(B, \ci, n_{i-1})$.
        row_means = [w.mean(-2) for w in W] 
        # Compute $\sqbra{\wtmat{1}_{:,\star},\cdots,\wtmat{L}_{:,\star}}$. Each has shape $(B, \ci, n_i)$.
        col_means = [w.mean(-1) for w in W] 
        # Compute $\sqbra{\wtmat{1}_{\star,\star},\cdots,\wtmat{L}_{\star,\star}}$ with shape $(B,\ci,L)$.
        rowcol_means = torch.stack([w.mean(dim=(-2, -1)) for w in W], -1)
        out_W = []
        for i, (w, a, b, b_prev, c, c_next, d) in enumerate(zip(
            W, self.A, self.B, self.B_prev, self.C, self.C_next, self.D
        )):
            # Calculates $\sum_s a^{i,s} \wt{s}{\star,\star}$.
            h1 = einsum(a, rowcol_means, 'co ci s, B ci s -> B co () ()')
            # Calculates $b^{i,i} \wt{i}{\star,k}$.
            h2 = einsum(b, row_means[i], 'co ci, b ci n_im1 -> B co () n_im1')
            if i > 0:
                # Calculates $b^{i,i-1} \wt{i-1}{k,\star}$.
                h2 += einsum(b_prev, col_means[i-1], 'co ci, B ci n_im1 -> B co () n_im1')
            # Calculates $c^{i,i}\wt{i}{j,\star}$.
            h3 = einsum(c, col_means[i], 'co ci, B ci n_i -> b co n_i ()')
            if i < self.L - 1:
                # Calculates $c^{i,i+1}\wt{i+1}{\star,j}$.
                h3 += einsum(c_next, row_col[i+1], 'co ci, B ci n_i -> B co n_i ()')
            # Calculates $d^{i} \wt{i}{jk}$.
            h4 = einsum(d, w, 'co ci, B ci n_i n_im1 -> B co n_i n_im1')
            # Calculates $H(W^{(i)})$ with shape $(B, c_o, n_i, n_{i-1})$.
            out_W.append(h1 + h2 + h3 + h4)
        return out_W
\end{minted}

\section{$\iogroup$-equivariant NF-Layer}
\label{appendix:equiv}
This section presents the construction of the $\iogroup$-equivariant layer $\hyperio:\calU\rightarrow \calU$ in detail. First, Sec.~\ref{subsec:NF-Layer-defn} gives the full definition of $\hyperio$ as a function of both weight-space and bias-space features. Sec.~\ref{subsec:general-NF-Layers} introduces general linear layers on $\calU$ parameterized by $\theta$. Sec.~\ref{sec:orbit-motivation} shows that \textbf{any} such linear layer must satisfy a certain parameter sharing to achieve equivariance. Sec.~\ref{subsec:parameter-sharing} derives this $\iogroup$-equivariant parameter sharing on $\theta$. Finally, Sec.~\ref{subsec:proof} shows how the parameter sharing reduces the general linear layers into our equivariant NF-Layer definition.

\subsection{Full definition}
\label{subsec:NF-Layer-defn}
We present the full definition of the $\iogroup$-equivariant NF-Layer $\hyperio:\calU \rightarrow \calU$, which completes Eq.~\ref{eq:layer-simple} by also including the biases. The layer processes a set of weights $U=(W,v)$ and outputs arrays $(Y, z)$ with the same dimensions. The learnable parameters are in \blue{blue}:
\begin{align}
    \label{eq:layer-full}
    \hyperio(U) \coloneqq &(Y, z) \in \calW \times \calV \\ \nonumber
    Y^{(i)}_{jk} \coloneqq  &\sum_s \paren{\blue{\aaa{i,s}{\vartheta}} \wt{s}{\star,\star} + \blue{\aaa{i,s}{\phi}}\bs{s}{\star}}
    + \blue{\bb{i,i}{\vartheta}} \wt{i}{\star,k} + \blue{\bb{i,i-1}{\vartheta}} \wt{i-1}{k,\star}
    \\ \nonumber
    &+ \blue{\bb{i}{\phi}}\bs{i}{j} + \blue{\cc{i,i}{\vartheta}} \wt{i}{j,\star} + \blue{\cc{i,i+1}{\vartheta}} \wt{i+1}{\star,j}
    + \blue{\cc{i}{\phi}}\bs{i-1}{k} + \blue{\dd{i}{\vartheta}} \wt{i}{jk}  \\ \nonumber
    z^{(i)}_j \coloneqq &\sum_s \paren{\blue{\aaa{i,s}{\varphi}}\wt{s}{\star,\star} + \blue{\aaa{i,s}{\psi}}\bs{s}{\star}} + \blue{\bb{i,i}{\varphi}}\wt{i}{j,\star} + \blue{\bb{i,i+1}{\varphi}}\wt{i+1}{\star,j} + \blue{\bb{i}{\psi}}\bs{i}{j},
\end{align}
where $\star$ denotes summation or averaging over a dimension.

\subsection{General NF-Layers}
\label{subsec:general-NF-Layers}
\begin{table}
\centering
\begin{tabular}{c | cccc}
  & $s=i-1$ & $s=i$ & $s=i+1$ & other $s$ \\
  \hline
  $\idx{\vartheta}{ijk}{spq}$ &
  $\begin{cases}
  \aaa{i,i-1}{\vartheta} & k \neq p \\
  \bb{i,i-1}{\vartheta} & k=p
  \end{cases}$
  & $\begin{cases}
  \aaa{i,i}{\vartheta} & j\neq p, k\neq q \\
  \bb{i,i}{\vartheta} & j=p, k\neq q \\
  \cc{i,i}{\vartheta} & j\neq p, k=q \\
  \dd{i}{\vartheta} & j=p, k=q
  \end{cases}$
  & $\begin{cases}
  \aaa{i,i+1}{\vartheta} & j \neq q\\
  \cc{i,i+1}{\vartheta} & j = q
  \end{cases}$
  & $ \aaa{i,s}{\vartheta}$ \\
\end{tabular}\vspace{0.4cm}
\begin{tabular}{c | ccc}
    & $s=i-1$ & $s=i$ & other $s$ \\
    \hline
    $\idx{\phi}{ijk}{sp}$
    & $ \case{\aaa{i,i-1}{\phi} & k \neq p \\ \bb{i,i-1}{\phi} & k = p}$
    & $ \case{\aaa{i,i}{\phi} & j \neq p \\ \bb{i,i}{\phi} & j = p}$
    & $ \aaa{i,s}{\phi}$
\end{tabular}\vspace{0.4cm}
\begin{tabular}{c | ccc}
    & $s=i$ & $s=i+1$ & other $s$ \\
    \hline
    $\idx{\varphi}{ij}{spq}$
    & $ \case{\aaa{i,i-1}{\varphi} & j \neq p \\ \bb{i,i-1}{\varphi} & j = p}$
    & $ \case{\aaa{i,i}{\varphi} & j \neq q \\ \bb{i,i}{\varphi} & j = q}$
    & $ \aaa{i,s}{\varphi}$
\end{tabular}\\
\vspace{0.4cm}
\begin{tabular}{c | cc}
    & $s=i$ & other $s$ \\
    \hline
    $\idx{\psi}{ij}{sp}$
    & $ \case{\aaa{i,i}{\psi} & j \neq p \\ \bb{i}{\psi} & j = p}$
    & $ \aaa{i,s}{\psi}$
\end{tabular}
\caption{$\iogroup$-equivariant parameter sharing for linear maps $\ve(U) \mapsto \theta \ve(U)$. Parameter sharing is a system of constraints on the entries of $(\vartheta, \phi, \varphi, \psi) = \theta$. Each table is organized by the layer indices $(i,s)$. For example, the first table says that for any $(i,s)$ where $s=i-1$, we constrain $\idx{\vartheta}{ijk}{spq} = \idx{\vartheta}{i,j',k'}{s,p',q'} = \aaa{i,i-1}{\vartheta}$ for any $j,k,p,q$ and $j',k',p',q'$ where $k\neq p$ and $k' \neq p'$. After parameter sharing, we observe that there are only a constant number of free parameters for each $(i,s)$ pair, adding up to $O\paren{L^2}$ parameters total.}
\label{table:theta-constraint}
\end{table}
To arrive at Eq.~\ref{eq:layer-full}, we begin by considering linear NF-Layers $T(\cdot; \theta):\calU \rightarrow \calU$ parameterized by $\theta \in \Theta$. If we flatten the input $U=(W,v)$ into a vector $\ve(U)\in\R^{\dim(\calU)}$, then the NF-Layer would be a matrix-vector product $T(\cdot,\theta): \ve(U) \mapsto \theta \ve(U)$ for square matrix $\theta\in\R^{\dim(\calU)\times\dim(\calU)}$.

For our purposes, it is sometimes convenient to distinguish layer, row, and column indices of entries in $U$ without any flattening, so we split the parameters $\theta=(\vartheta,\phi,\varphi,\psi)$ and write $T:\calU\rightarrow\calU$ in the form:
\begin{align}
    \label{eq:NF-Layer}
    T(U; \theta) &\coloneqq (Y(U), z(U)) \quad \in \calW \times \cal V=\calU\\
    Y(U)^{(i)}_{jk} &\coloneqq \sum_{s=1}^L\sum_{p=1}^{n_s}\sum_{q=1}^{n_{s-1}}\idx{\vartheta}{ijk}{spq} \wt{s}{pq} + \sum_{s=1}^L\sum_{p=1}^{n_s}\idx{\phi}{ijk}{sp} \bs{s}{p}\\
    z(U)^{(i)}_j &\coloneqq \sum_{s=1}^L\sum_{p=1}^{n_s}\sum_{q=1}^{n_{s-1}} \idx{\varphi}{ij}{spq} \wt{s}{pq} + \sum_{s=1}^L\sum_{p=1}^{n_s}\idx{\psi}{ij}{sp}\bs{s}{p}.
\end{align}
Since we can equivalently flatten this operation into the matrix-vector product $\ve(U) \mapsto \theta \ve(U)$, we introduce the notation $U_\alpha$ to identify individual entries of $U$. Here $\alpha$ is a tuple of length two or three, for indexing into either a weight or bias. We denote the space of valid index tuples of $\calW$ and $\calV$ by $\WW$ and $\VV$, respectively, and define $\UU \coloneqq \WW \cup \VV$ as the combined index space of $\calU$. For example, if $\alpha=(i,j,k) \in \WW$, then $U_\alpha = \wt{i}{jk}$.

We can then define the index space $\II \coloneqq \UU \times \UU$ for parameters $\theta \in \Theta$. We use $\theta^\alpha_\beta$ to index an entry of $\theta$ with upper and lower indices $\sqbra{\alpha, \beta} \in \II$ . For example, if $\alpha=(i,j,k)\in \WW$ and $\beta=(s,p,q)\in \WW$, we have $\theta^{\alpha}_{\beta} = \idx{\vartheta}{ijk}{spq}$.

The indices $\alpha$ and $\beta$ correspond to rows and columns of the matrix $\theta$, respectively. Eq.~\ref{eq:NF-Layer} can be rewritten in the flattened form:
\begin{equation}
    \label{eq:NF-Layer-flattened}
    T(U;\theta)_\alpha = \sum_\beta \theta^\alpha_\beta U_\beta.
\end{equation}
Finally, we can re-express the action of $\iogroup$ on $\calU$ (Eq.~\ref{eq:action}) as an action on the index space $\UU$:
\begin{align}
    \label{eq:vartheta-action}
    \sigma (i,j,k) &= (i, \sigma_i(j), \sigma_{i-1}(k)) \quad (i,j,k) \in \WW\\
    \sigma (i,j) &= (i, \sigma_i(j)) \quad (i,j) \in \VV,
\end{align}
for any $\sigma \in \iogroup$. We extend this definition into an action of $\iogroup$ on $\II$:
\begin{equation}
\label{eq:theta-action}
\sigma\sqbra{\alpha, \beta} \coloneqq [\sigma \alpha, \sigma \beta], \quad \sqbra{\alpha,\beta} \in \UU\times\UU.
\end{equation}

\subsection{Equivariance and parameter sharing}
\label{sec:orbit-motivation}
We would like to find the constraints on $\theta$ that make the linear map $T(\cdot;\theta) : \ve(U) \mapsto \theta \ve(U)$ equivariant to $\iogroup$.

We can represent the action of $\sigma \in \iogroup$ on $\ve(U)$ by a matrix $P_\sigma \in \Set{0,1}^{\dim(U)\times\dim(U)}$. Equivariance requires that $P_\sigma\theta \ve(U) = \theta P_\sigma \ve(U)$ for any $\sigma \in \iogroup$. Since the input $U$ can be anything, we get the following constraint on $\theta$:
\begin{equation}
    \label{eq:matrix-constraint}
    P_\sigma \theta = \theta P_\sigma, \quad \forall \sigma \in \iogroup.
\end{equation}

When written out using indices $\alpha,\beta$, the constraint requires that for any $\sigma \in \iogroup$:
\begin{equation}
    \sqbra{P_\sigma \theta}^\alpha_\beta = \theta^{\sigma^{-1}(\alpha)}_{\beta}
    = \theta^{\alpha}_{\sigma(\beta)} = \sqbra{\theta P_\sigma}^\alpha_\beta.
\end{equation}
By relabeling $\alpha \gets \sigma^{-1}(\alpha)$, we can rewrite this condition $\theta^{\alpha}_\beta = \theta^{\sigma(\alpha)}_{\sigma(\beta)}$.
Hence for any linear $\iogroup$-equivariant map $T(\cdot,\theta):\calU \rightarrow \calU$, $\theta$ must share parameters within orbits under the action of $\iogroup$ on its indices $\alpha,\beta$ (Eq.~\ref{eq:theta-action}). In fact, this strategy was first proposed as a way of constructing equivariant layers by \citet[Prop 3.1]{ravanbakhsh2017equivariance}.

\subsection{$\iogroup$-equivariant parameter sharing}
\label{subsec:parameter-sharing}
We now derive the required parameter sharing conditions on $\theta$ to make $T(\cdot; \theta):\calU \rightarrow \calU$ equivariant to $\iogroup$. Our approach is to partition the parameters of $\theta$ into orbits under the $\iogroup$-action on its index space (Eq.~\ref{eq:theta-action}), and share parameters within an orbit.

The index space of $\theta$ is $\II = \UU \times \UU$. There are four subsets of $\II$:
\begin{enumerate}
    \item $\II^{WW}\coloneqq\WW \times \WW$: Contains $\sqbra{\alpha,\beta} = \sqbra{(i,j,k),(s,p,q)}$, indexing parameters $\idx{\vartheta}{ijk}{spq}$.
    \item $\II^{WV}\coloneqq\WW \times \VV$: Contains $\sqbra{\alpha,\beta} = \sqbra{(i,j,k),(s,p)}$, indexing parameters $\idx{\phi}{ijk}{sp}$.
    \item $\II^{VW}\coloneqq\VV \times \WW$: Contains $\sqbra{\alpha,\beta} = \sqbra{(i,j),(s,p,q)}$, indexing parameters $\idx{\varphi}{ij}{spq}$.
    \item $\II^{VV}\coloneqq\VV \times \VV$: Contains $\sqbra{\alpha,\beta} = \sqbra{(i,j),(s,p)}$, indexing parameters $\idx{\psi}{ij}{sp}$.
\end{enumerate}
Equivariant parameter sharing then amounts to partitioning $\II$ into orbits under $\iogroup$, and then sharing the corresponding parameters within each orbit.

Consider the block of indices $\II^{WW}=\WW \times \WW$, containing $\sqbra{\alpha, \beta} = \sqbra{(i,j,k),(s,p,q)}$ indexing parameters $\idx{\vartheta}{\alpha}{\beta}$. Since the $\iogroup$-action never changes the layer indices $(i,s)$, we can independently consider orbits within sub-blocks of indices
$\II^{WW}_{i,s} = \Set{\sqbra{(i,j,k),(s,p,q)} | \forall j,k,p,q}$.
The number of orbits within each sub-block $\II^{WW}_{i,s}$ depends on the relationship between the layer indices $i$ and $s$: they are either the same layer ($s=i$), they are adjacent ($s=i-1$ or $s=i+1$), or they are non-adjacent ($s\notin \{i-1,i,i+1\}$). We now analyze the orbits of sub-blocks for a few cases.

If $s=i-1$, then choose any two indices $\sqbra{\alpha^{(1)},\beta^{(1)}},\sqbra{\alpha^{(2)},\beta^{(2)}}\in\II^{WW}_{i,s}$ where the first satisfies $p\neq k$ and the second satisfies $p = k$. Then the orbits of each index are:
\begin{align}
    \orb\paren{\sqbra{\alpha^{(1)},\beta^{(1)}}}=\Set{[(i,j,k),(s,p,q)] \mid \forall j,k,p,q : p \neq k}\\
    \orb\paren{\sqbra{\alpha^{(2)},\beta^{(2)}}} = \Set{[(i,j,k),(s,p,q)] \mid \forall j,k,p,q : p = k}.
\end{align}
We see that these two orbits actually partition the entire sub-block of indices $\II^{WW}_{i,s}$, with each orbit characterized by whether or not $p=k$. We introduce the parameters $\aaa{i,i-1}{\vartheta}$ (for the first orbit) and $\bb{i,i-1}{\vartheta}$ (for the second orbit). Under equivariant parameter sharing, all parameters of $\vartheta$ corresponding $\II^{WW}_{i,s}$ are equal to either $\aaa{i,i-1}{\vartheta}$ or $\bb{i,i-1}{\vartheta}$, depending on whether $p=k$ or $p\neq k$.

If $s=i+1$, we instead choose any two indices where the first satisfies $j \neq q$ and the second satisfies $j = q$. Then the sub-block of indices $\II^{WW}_{i,s}$ is again partitioned into two orbits:
\begin{equation}
    \Set{[(i,j,k),(s,p,q)] | \forall j,k,p,q : j \neq q}, \text{ and } \Set{[(i,j,k),(s,p,q)] | \forall j,k,p,q : j = q}
\end{equation}
depending on the condition $j=q$. We name two parameters $\aaa{i,i+1}{\vartheta}$ and $\cc{i,i+1}{\vartheta}$ for this sub-block, with one for each orbit.

We can repeat this process for sub-blocks of $\II^{WW}$ where $i=s$ and $s\notin \{i-1,i,i+1\}$, as well as for the other three blocks of $\II$.
Table~\ref{table:theta-constraint} shows the complete parameter sharing constraints on $\theta$ resulting from partitioning all possible sub-blocks into orbits.

\textbf{Number of parameters.} We also note that every layer pair $(i,s)$ introduces only a constant number of parameters: the number of parameters in each cell of Table~\ref{table:theta-constraint} has no dependence on the input, output, or hidden dimensions of $\calU$. Hence the number of distinct parameters after parameter sharing simply grows with the number of layer pairs, i.e. $O\paren{L^2}$. 

\subsection{Equivalence to equivariant NF-Layer definition}
\label{subsec:proof}
All that remains is to show that the map $T(\cdot; \theta): \calU \rightarrow \calU$ with $\iogroup$-equivariant parameter sharing (Table~\ref{table:theta-constraint}) is equivalent to the NF-Layer $\hyperio$ we defined in Eq.~\ref{eq:layer-full}.

Consider a single term from Eq.~\ref{eq:NF-Layer} where $s=i-1$. Substituting using the constraints of Table~\ref{table:theta-constraint}, we simplify:
\begin{align}
    \sum_{p,q}\vartheta^{i,j,k}_{i-1,p,q}W^{(i-1)}_{pq} &= a^{i,i-1} \sum_q \sum_{k\neq p} W^{(i-1)}_{p,q} + b^{i,i-1}\sum_q W^{(i-1)}_{k,q}\\ \nonumber
    &= a^{i,i-1} \wt{i-1}{\star,\star} + (b^{i,i-1} - a^{i,i-1}) \wt{i-1}{k,\star}.
\end{align}
We can then reparameterize $b^{i,i-1} \gets b^{i,i-1} - a^{i,i-1}$, resulting in two terms that appear in Eq.~\ref{eq:layer-full}. We can simplify every term of Eq.~\ref{eq:NF-Layer} in a similar manner using the parameter sharing of Table~\ref{table:theta-constraint}, reducing the general layer to the $\iogroup$-equivariant NF-Layer.

\section{NF-Layers for the HNP setting}
\label{appendix:nonio-theory}
\subsection{Equivariant NF-Layer}
Because an expression for the $\group$-equivariant NF-Layer analogous to Eq.~\ref{eq:layer-full} would be unwieldy, we instead define the layer in terms of its parameter sharing (Tables \ref{table:hnp-vartheta}-\ref{table:hnp-psi}) on $\theta$.

We can derive HNP-equivariant parameter sharing of $\theta$ using a similar strategy to Sec.~\ref{subsec:parameter-sharing}: we partition the index spaces $\II^{WW},\II^{WV},\II^{VW},\II^{VV}$ into orbits under the action of $\group$, and share parameters within each corresponding orbit of $\vartheta,\phi,\varphi,\psi$. 
The resulting parameter sharing is different from the NP-setting because while the action of $\iogroup$ on $\calU$ could permute the rows and columns of every weight and bias, the action of $\group$ on $\calU$ does not affect the columns of $\wtmat{1}$ or the rows of $\wtmat{L},\bsvec{L}$, which correspond to input and output dimensions (respectively).

The orbits are again analyzed within sub-blocks defined by the values of the layer indices $(i,s)$. As with the NP setting, there are broadly four types of sub-blocks based on whether $i=s$, $s=i-1$, $s=i+1$, or $s\notin \{i-1,i,i+1\}$. However, there are now additional considerations based on whether $i$ or $s$ is an input or output layer. For example, consider the sub-block of $\II^{WW}$ where $i=s=1$, which we denote $\II^{WW}_{1,1}$. The action on the indices in this sub-block can be written $\sigma \sqbra{\alpha,\beta} = \sqbra{(1,\sigma_1(j),k),(1,\sigma_1(p),q)}$. Importantly, the column indices $k,q$ are never permuted since they correspond to the input layer. We see that $\II^{WW}_{1,1}$ contains two orbits \textit{for each} $k\in\range{1..n_0}$ and $q\in\range{1..n_0}$, with the two orbits characterized by whether or not $j=p$. Hence we have $2n_0^2$ orbits and Table~\ref{table:hnp-vartheta} introduces $2n_0^2$ parameters $\Set{\aaa{1,1,k,q}{\vartheta},\bb{1,1,k,q}{\vartheta} | k,q \in \range{1..n_0}}$ for this sub-block of parameters.

Now consider another sub-block of $\II^{WW}$ where $1 < i=s < L$. Now the action of $\group$ on indices in this sub-block can be written $\sigma \sqbra{\alpha, \beta} = \sqbra{(i,\sigma_i(j),\sigma_{i-1}(k)), (i, \sigma_i(p), \sigma_{i-1}(q))}$. Then we have a total of two orbits characterized by whether or not $k=p$, rather than $2n_0^2$ orbits for the $i=1$ case.
Tables~\ref{table:hnp-vartheta}-\ref{table:hnp-psi} present the complete parameter sharing for each of $\vartheta,\phi,\varphi,\psi$, resulting from analyzing every possible orbit within any sub-block of $\II^{WW},\II^{WV},\II^{VW},\II^{VV}$.

\begingroup
\renewcommand{\arraystretch}{1.5}
\begin{table}
\begin{tabular}{c !{\vrule width 1.5pt} c|c|c}
  \multicolumn{4}{c}{$\idx{\vartheta}{ijk}{spq}$} \\
  \hhline{-|---}
  \hhline{-|---}
  \hhline{-|---}
  \multirow{3}{*}{$s=i-1$} & $i=2$ & $2<i<L$ & $i=L$ \\
  \cline{2-4}
    & $\begin{cases}
        \aaa{2,1,q}{\vartheta} & k \neq p \\
        \bb{2,1,q}{\vartheta}  & k = p
    \end{cases}$
    & $\begin{cases}
        \aaa{i,i-1}{\vartheta} & k\neq p \\
        \bb{i,i-1}{\vartheta} & k = p \\
    \end{cases}$
    & $\begin{cases}
        \aaa{L,L-1,j}{\vartheta} & k \neq p \\
        \bb{L,L-1,j}{\vartheta} & k = p
    \end{cases}$\\
  \hhline{-|---}
  \hhline{-|---}
  \hhline{-|---}
  \multirow{5}{*}{$s=i$} & $i=1$ & $1<i<L$ & $i=L$ \\
  \cline{2-4}
    &$\begin{cases}
        \aaa{1,1,k,q}{\vartheta} & j \neq p\\
        \bb{1,1,k,q}{\vartheta} & j = p\\
    \end{cases}$
    & $\begin{cases}
        \aaa{i,i}{\vartheta} & j\neq p, k\neq q\\
        \bb{i,i}{\vartheta} & j=p, k \neq q\\
        \cc{i,i}{\vartheta} & j\neq p, k=q \\
        \dd{i,i}{\vartheta} & j=p,k=q\\
    \end{cases}$
    & $ \begin{cases}
        \aaa{L,L,j,p}{\vartheta} & k \neq q \\
        \cc{L,L,j,p}{\vartheta} & k = q
    \end{cases}$ \\
  \hhline{-|---}
  \hhline{-|---}
  \hhline{-|---}
  \multirow{3}{*}{$s=i+1$} & $i=1$ & $1<i<L-1$ & $i=L-1$ \\
  \cline{2-4}
    & $ \begin{cases}
        \aaa{1,2,k}{\vartheta} & j \neq q\\
        \cc{1,2,k}{\vartheta} & j = q\\
    \end{cases}$
    & $ \begin{cases}
        \aaa{1,2}{\vartheta} & j \neq q\\
        \cc{1,2}{\vartheta} & j = q\\
    \end{cases}$
    & $ \begin{cases}
        \aaa{L-1,L,p}{\vartheta} & j \neq q\\
        \cc{L-1,L,p}{\vartheta} & j = q\\
    \end{cases}$ \\
  \hhline{-|---}
  \hhline{-|---}
  \hhline{-|---}
  \multirow{6}{*}{other $s$} & $i=1,1<s<L$ & $i=1,s=L$ & $1<i<L,s=L$ \\
  \cline{2-4}
    & $\aaa{1,s,k}{\vartheta}$
    & $\aaa{1,L,k,p}{\vartheta}$
    & $\aaa{i,L,p}{\vartheta}$ \\
  \hhline{~|---}
  \hhline{~|---}
  & $1<i<L,s=1$ & $i=L,s=1$ & $i=L,1<s<L$ \\
  \cline{2-4}
    & $ \aaa{i,1,q}{\vartheta}$
    & $ \aaa{L,1,j,q}{\vartheta}$
    & $ \aaa{L,s,j}{\vartheta}$ \\
  \hhline{~|---}
  \hhline{~|---}
  & $1<i<L,1 < s < L$ & & \\
  \cline{2-2}
  & $\aaa{i,s}{\vartheta}$ & & \\
\end{tabular}
\caption{HNP-equivariant parameter sharing on $\vartheta\subseteq \theta$, corresponding to the NF-Layer $\hyper:\calU\rightarrow\calU$.}
\label{table:hnp-vartheta}
\end{table}

\begin{table}
\begin{tabular}{c !{\vrule width 1.5pt} c|c|c}
  \multicolumn{4}{c}{$\idx{\phi}{ijk}{sp}$} \\
  \hhline{-|---}
  \hhline{-|---}
  \hhline{-|---}
  \multirow{3}{*}{$s=i-1$} & & $1 < i < L$ & $i=L$ \\
  \cline{3-4}
    & & $\begin{cases}
        \aaa{i,i-1}{\phi} & k\neq p \\
        \bb{i,i-1}{\phi} & k = p \\
    \end{cases}$
    & $\begin{cases}
        \aaa{L,L-1,j}{\phi} & k \neq p \\
        \bb{L,L-1,j}{\phi} & k = p
    \end{cases}$\\
  \hhline{-|---}
  \hhline{-|---}
  \hhline{-|---}
  \multirow{3}{*}{$s=i$} & $i=1$ & $1<i<L$ & $i=L$ \\
  \cline{2-4}
    &$\begin{cases}
        \aaa{1,1,k}{\phi} & j \neq p\\
        \bb{1,1,k}{\phi} & j = p\\
    \end{cases}$
    & $\begin{cases}
        \aaa{i,i}{\phi} & j\neq p \\
        \bb{i,i}{\phi} & j=p \\
    \end{cases}$
    & $ \bb{L,L,j,p}{\phi}$ \\
  \hhline{-|---}
  \hhline{-|---}
  \hhline{-|---}
  \multirow{4}{*}{other $s$} & $i=1,1<s<L$ & $i=1,s=L$ & $1<i<L,s=L$ \\
  \cline{2-4}
    & $\aaa{1,s,k}{\phi}$
    & $\aaa{1,L,k,p}{\phi}$
    & $\aaa{i,L,p}{\phi}$ \\
  \hhline{~|---}
  \hhline{~|---}
  & $1<i<L,1 \leq s < L$ & $i=L,1 \leq s < L$ & \\
  \cline{2-3}
    & $ \aaa{i,s}{\phi}$
    & $ \aaa{L,s,j}{\phi}$ & \\
\end{tabular}
\caption{HNP-equivariant parameter sharing on $\phi \subset \theta$, corresponding to the NF-Layer $\hyper:\calU\rightarrow\calU$.}
\label{table:hnp-phi}
\end{table}

\begin{table}
\begin{tabular}{c !{\vrule width 1.5pt} c|c|c}
  \multicolumn{4}{c}{$\idx{\varphi}{ij}{spq}$} \\
  \hhline{-|---}
  \hhline{-|---}
  \hhline{-|---}
  \multirow{3}{*}{$s=i$} & $i=1$ & $1<i<L$ & $i=L$ \\
  \cline{2-4}
    &$\begin{cases}
        \aaa{1,1,q}{\varphi} & j \neq p\\
        \bb{1,1,k}{\varphi} & j = p\\
    \end{cases}$
    & $\begin{cases}
        \aaa{i,i}{\varphi} & j\neq p \\
        \bb{i,i}{\varphi} & j=p \\
    \end{cases}$
    & $ \bb{L,L,j,p}{\varphi}$ \\
  \hhline{-|---}
  \hhline{-|---}
  \hhline{-|---}
  \multirow{3}{*}{$s=i+1$} & & $1\leq i < L-1$ & $i=L-1$ \\
  \cline{3-4}
    & & $ \begin{cases}
        \aaa{i,i+1}{\varphi} & j \neq p \\
        \bb{i,i+1}{\varphi} & j = p
    \end{cases}$
    & $ \begin{cases}
        \aaa{L-1,L,p}{\varphi} & j \neq p \\
        \bb{L-1,L,p}{\varphi} & j = p
    \end{cases}$ \\
  \hhline{-|---}
  \hhline{-|---}
  \hhline{-|---}
  \multirow{4}{*}{other $s$} & $1\leq i < L,s=1$ & $1\leq i < L, 1 < s < L$ & $1\leq i < L,s=L$ \\
  \cline{2-4}
    & $\aaa{i,s,q}{\varphi}$
    & $\aaa{i,s}{\varphi}$
    & $\aaa{i,L,p}{\varphi}$ \\
  \hhline{~|---}
  \hhline{~|---}
  & $i=L,s=1$ & $i=L,1 < s < L$ & \\
  \cline{2-3}
    & $ \aaa{L,1,j,q}{\varphi}$
    & $ \aaa{L,s,j}{\varphi}$ & \\
\end{tabular}
\caption{HNP-equivariant parameter sharing on $\varphi \subset \theta$, corresponding to the NF-Layer $\hyper:\calU\rightarrow\calU$.}
\label{table:hnp-varphi}
\end{table}

\begin{table}
\begin{tabular}{c !{\vrule width 1.5pt} c|c|c}
  \multicolumn{4}{c}{$\idx{\psi}{ij}{sp}$} \\
  \hhline{-|---}
  \hhline{-|---}
  \hhline{-|---}
  \multirow{3}{*}{$s=i$} & & $1\leq i<L$ & $i=L$ \\
  \cline{3-4}
    & & $\begin{cases}
        \aaa{i,i}{\psi} & j\neq p \\
        \bb{i,i}{\psi} & j=p \\
    \end{cases}$
    & $ \bb{L,L,j,p}{\psi}$ \\
  \hhline{-|---}
  \hhline{-|---}
  \hhline{-|---}
  \multirow{2}{*}{other $s$} & $1\leq i < L,s=L$ & $i=L, 1 \leq s < L$ & $1 \leq i < L, 1\leq s < L$ \\
  \cline{2-4}
    & $\aaa{i,L,p}{\psi}$
    & $\aaa{L,s,j}{\psi}$
    & $\aaa{i,s}{\psi}$ \\
\end{tabular}
\caption{HNP-equivariant parameter sharing on $\psi \subset \theta$, corresponding to the NF-Layer $\hyper:\calU\rightarrow\calU$.}
\label{table:hnp-psi}
\end{table}
\endgroup

\subsection{Invariant NF-Layer}
While the NP-invariant NF-Layer sums over the rows and columns of every weight and bias, under HNP assumptions there is no need to sum over the columns of $\wtmat{1}$ (inputs) or the rows of $\wtmat{L}, \bsvec{L}$ (outputs).  So the HNP invariant NF-Layer $\inv:\calU\rightarrow\R^{2L + n_0 + 2n_L}$ is defined:
\begin{equation}
    \label{eq:pooling-nonio}
    \inv(U) = \paren{\invio(U), \wt{1}{\star,:}, \wt{L}{:,\star}, \bsvec{L}},
\end{equation}
where $\wt{1}{\star,:}$ and $\wt{L}{:,\star}$ denote summing over only the rows or only the columns of the matrix, respectively. Note that $\inv$ satisfies $\group$-invariance without satifying $\iogroup$-invariance.

\section{Additional experimental details}
\subsection{Predicting generalization}

The model we use consists of three equivariant NF-Layers with 16, 16, and 5 channels respectively. We apply ReLU activations after each linear NF-Layer. The resulting weight space features are passed into an invariant NF-Layer with mean pooling. The output of the invariant NF-Layer is flattened and projected to $\R^{1,000}$. The resulting vector is then passed through an MLP with two hidden layers, each with 1,000 units and ReLU activations. The output is linearly projected to a scalar and passed through a sigmoid function.
Since the output of the model can be interpreted as a probability, we train the model with binary cross-entropy with hyperparameters outlined in Table~\ref{tab:hparam-cnn}. The model is trained for 50 epochs with early stopping based on $\tau$ on the validation set, which takes $1$ hour on a Titan RTX GPU.

\begin{table}[h]
    \centering
    \begin{tabular}{lr}
    \toprule
       \textbf{Name}
       & \textbf{Values}   \\
      \midrule
      Optimizer & Adam   \\
      Learning rate & $0.001$   \\
      Batch size & 8 \\
      Loss & Binary cross-entropy \\
      Epoch & 50 \\
      \bottomrule
    \end{tabular}
    \caption{Hyperparameters for predicting generalization on Small CNN Zoo.}
    \label{tab:hparam-cnn}
\end{table}

\subsection{Predicting ``winning ticket'' masks from initialization}
Concretely, the encoder learns the posterior distribution $q_\theta(Z \mid U_0, M)$ where $Z \in \R^{\dim(\calU) \times C}$ is the latent variable for the winning tickets and $C$ is the number of latent channels. The decoder learns $p_\theta(M \mid U_0, Z)$, and both encoder and decoder are implemented using our equivariant NF-Layers. For the prior $p(Z)$ we choose the isometric Gaussian distribution, and train using the evidence lower bound (ELBO):
\begin{align*}
    \mathcal{L}_{\theta}(M, U_0) = \E_{z \sim q_\theta(\cdot \mid U_0, M)} \Big[ \ln p_\theta(M \mid U_0, z) \Big] - \text{D}_\text{KL}\Big( q_\theta(\cdot \mid U_0, M) \mid \mid p(\cdot) \Big).
\end{align*}
The initialization and sparsity mask are concatenated so the input to the encoder $q_\theta$ is $(U,M) \in \R^{\dim(\calU) \times 2}$. After the bottleneck, we concatenate the latent variables and the original mask along the channels, i.e. the decoder input is $(U_0, Z) \in \R^{\dim(\calU) \times (C + 1)}$.

The first dataset uses three-layer MLPs with 128 hidden units trained on MNIST and the second uses CNNs with three convolution layers (128 channels) and 2 fully-connected layers trained on CIFAR-10. In each dataset, we include 400 pairs for training and hold out 50 for evaluation.
The hyperparameter details are in Table~\ref{tab:hparam-lth}. The encoder and decoder models contain 4 equivariant NF-Layers with 64 hidden channels within each layer. The latent variable is 5 dimensions. Training takes 5H on a Titan RTX GPU.

\begin{table}[h]
    \centering
    \begin{tabular}{lr}
    \toprule
       \textbf{Name}
       & \textbf{Values}   \\
      \midrule
      Optimizer &  Adam  \\
      Learning rate &  $1\times 10^{-3}$  \\
      Batch size & [4, 8] \\
      Epoch & 200 \\
      \bottomrule
    \end{tabular}
    \caption{Hyperparameters for predicting LTH on MNIST and CIFAR-10.}
    \label{tab:hparam-lth}
\end{table}

\subsection{Classifying INRs}

We use SIREN~\citep{sitzmann2020implicit} for our INRs of CIFAR, FashionMNIST, and MNIST. For the SIREN models, we used a three-layer architecture with 32 hidden neurons in each layer. We trained the SIRENs for 5,000 steps using Adam optimizer with a learning rate of $5\times 10^{-5}$. Datasets were split into 45,000 training images, 5,000 validation images, and 10,000 (MNIST, CIFAR) or 20,000 (FashionMNIST) test images.
We trained 10 copies (MNIST, FashionMNIST) or 20 copies (CIFAR-10) of SIRENs on each training image with different initializations, and a single SIREN on each validation and test image. No additional data augmentation was applied.
For 3D shape classification, we adopt the same protocol introduced in~\citep{2023inr2vec}, and we train each SIREN to fit the Unsigned Distance Function (UDF) value of points sampled around a shape. Each SIREN is composed of a single hidden layer with 128 neurons. We use Adam as an optimizer and we train for 1,000 steps.

We also trained neural functionals with three equivariant NF-Layers + ReLU activations, each with 512 channels, followed by invariant NF-Layers (mean pooling) and a three-layer MLP head with 1,000 hidden units and ReLU activation. Dropout was applied to the MLP head only. For the \genours IO-encoding, we used sinusoidal position encoding with a maximum frequency of 10 and 6 frequency bands (dimension 13). The training hyperparameters are shown in Table~\ref{tab:hparam-inr}, and training took $\sim 4$H on a Titan RTX GPU.

\begin{table}[h]
    \centering
    \begin{tabular}{lr}
    \toprule
       \textbf{Name}
       & \textbf{Values}   \\
      \midrule
      Optimizer &  Adam  \\
      Learning rate &  $1\times 10^{-4}$  \\
      Batch size & 32 \\
      Training steps & $2\times 10^{5}$  \\
      MLP dropout & $0.5$\\
      \bottomrule
    \end{tabular}
    \caption{Hyperparameters for classifying INRs on MNIST and CIFAR-10 using neural functionals.}
    \label{tab:hparam-inr}
\end{table}

We also experimented with larger MLPs (4,000 and 8,000 hidden units per layer) that have parameter counts comparable to those of the NFNs, but found that it did not significantly increase test accuracy, as shown in Table \ref{tab:mlp-inr-classification}.
\begin{table}
\centering
    \caption{Classification train and test accuracies (\%) for datasets of implicit neural representations (INRs) of either MNIST or CIFAR-10. Although permutation augmentations slightly increase performance by reducing overfitting, even the larger MLPs are unable to robustly classify INRs. Uncertainties indicate standard error over three runs.}
\begin{tabular}{ccrrrr}
      \toprule
     & & MLP-4000 & $\text{MLP-4000}_\text{Aug}$ & MLP-8000 & $\text{MLP-8000}_\text{Aug}$\\
      \midrule
      \multirow{2}{*}{\centering CIFAR-10} & Train & $30.4\pm0.521$ & $20.5\pm0.333$ & $35.9\pm0.868$ & $18.1\pm0.347$  \\
        & Test & $17.1 \pm 0.120$ & $19.3 \pm 0.325$ & $17.3 \pm 0.280$ & $19.6 \pm 0.060$ \\
        \midrule
      \multirow{2}{*}{\centering MNIST} & Train & $72.6\pm1.39$ & $19.4\pm1.39$ & $77.8\pm1.74$ & $20.3\pm3.30$\\
        & Test & $15.5 \pm 0.090$ & $21.1 \pm 0.010$ & $15.8 \pm 0.014$ & $21.3 \pm 0.075$\\
      \bottomrule
    \end{tabular}
    \label{tab:mlp-inr-classification}
\end{table}

\subsection{Weight space style editing}

For weight space editing, we use the same INRs as the ones used for classification but we do not augment the dataset with additional INRs. Let $U_i$ be the INR weights for the $i^\text{th}$ image and $\inr(x, y;U)$ be the output of the INR parameterized by $U$ at coordinates $(x,y)$. We edit the INR weights $U_i' = U_i + \gamma \cdot \textsc{NFN}(U_i)$, and $\gamma$ is a learned scalar initialized to $0.01$. Letting $f_i(x, y)$ be the pixel values of the ground truth edited image (obtained from image-space processing), the objective is to minimize mean squared error:
\begin{align}
    \mathcal{L}\paren{\textsc{NFN}} = \frac{1}{N \cdot d^2} \sum_{i=1}^N \sum_{x,y}^{d} \| \inr\paren{x, y;U'} - f_i(x,y)\|_2^2.
\end{align}
Note that since the SIREN itself is differentiable, the loss can be directly backpropagated through $U'$ to the parameters of the \genours{}.

The neural functionals contain 3 equivariant NF-Layers with 128 channels, one invariant NF-Layer (mean pooling) followed by 4 linear layers with 1,000 hidden neurons. Every layer uses ReLU activation. The training hyperparameters can be found in Table~\ref{tab:hparam-editing}, and training takes $\sim 1$ hour on a Titan RTX GPU.

\begin{table}[h]
    \centering
    \begin{tabular}{lr}
    \toprule
       \textbf{Name}
       & \textbf{Values}   \\
      \midrule
      Optimizer &  Adam  \\
      Learning rate &  $1\times 10^{-3}$  \\
      Batch size & 32 \\
      Training steps & $5\times 10^{4}$  \\
      \bottomrule
    \end{tabular}
    \caption{Hyperparameters for weight space style editing using neural functionals.}
    \label{tab:hparam-editing}
\end{table}

\section{Additional experiments and analysis}
\begin{table}
    \centering
    \caption{Classification train and test accuracies (\%) for implicit neural representations of MNIST, FashionMNIST, and CIFAR-10. Our equivariant NFNs outperform the MLP baselines, even when the MLP has permutation augmentations to encourage invariance. Uncertainties indicate standard error over three runs.}
        \begin{tabular}{ccrrrr}
      \toprule
       & & \fullours & \ours & 
\mlp & \mlpaug \\
      \midrule
      \multirow{2}{*}{\centering CIFAR-10} & Train & $75.5\pm0.810$ & $66.0\pm0.694$ & $23.7\pm2.39$ & $19.1\pm1.75$ \\
        & Test & $44.1 \pm 0.471$ & $\mathbf{46.6 \pm 0.072}$ & $16.9 \pm 0.250$ & $18.9 \pm 0.432$ \\
        \midrule
      \multirow{2}{*}{\centering MNIST} & Train & $94.9\pm0.579$ & $95.0\pm0.115$ & $42.4\pm2.44$ & $20.5\pm0.401$ \\
        & Test & $92.5 \pm 0.071$ & $\mathbf{92.9 \pm 0.218}$ & $14.5 \pm 0.035$ & $21.0 \pm 0.172$ \\
        \midrule
      \multirow{2}{*}{\centering FashionMNIST} & Train & $82.3\pm2.78$ & $81.8\pm0.868$ & $44.5\pm2.17$ & $14.9\pm1.45$ \\
        & Test & $72.7\pm1.53$ & $\mathbf{75.6\pm1.07}$ & $12.5\pm0.111$ & $15.9\pm0.181$ \\
      \bottomrule
    \end{tabular}
    \label{tab:inr-classification-full}
\end{table}
\begin{table}
    \centering
    \caption{Classification train and test accuracies (\%) for datasets of implicit neural representations (INRs) of either ShapeNet-10~\citep{shapenet2015} or ScanNet-10~\citep{dai2017scannet} Our equivariant NFNs outperform the MLP baselines and recent non-equivariant methods such as inr2vec \citep{2023inr2vec}. Uncertainties indicate standard error over three runs.}
    \scalebox{0.92}{
        \begin{tabular}{ccrrrrr}
      \toprule
       & & \fullours & \ours & 
\mlp & \mlpaug & inr2vec\citep{2023inr2vec} \\
      \midrule
      \multirow{2}{*}{\centering ShapeNet-10} & Train & $100\pm0.0$ & $100.0\pm0.0$ & $100.0\pm0.0$ & $34.0\pm0.0$ & $99.0\pm0.0$\\
        & Test & $86.9\pm0.860$ & $\mathbf{88.7\pm0.461}$ & $25.4\pm0.121$ & $33.8\pm0.126$ & $39.1\pm 0.385$\\
        \midrule
      \multirow{2}{*}{\centering ScanNet-10} & Train & $100.0\pm0.0$ & $100.0\pm0.0$ & $100.0\pm0.0$ & $42.7\pm0.012$ & $93.8\pm0.090$ \\
        & Test & $64.1\pm0.572$ & $\mathbf{65.9\pm1.10}$ & $ 32.9\pm0.351$ & $45.5\pm0.126 $ & $38.2\pm0.409$ \\
      \bottomrule
    \end{tabular}}
    \label{tab:inr_3d-classification-full}
\end{table}

\subsection{Interpreting learned lottery ticket masks}
\label{appendix:lth-analysis}
We further analyze the behavior of \ptours on lottery ticket mask prediction by plotting the mask score predicted for a given initialization value at each layer. To make the visualization clear we train \ptours on MLP mask prediction without layer norm, which can be viewed as a scalar function of the initialization $f^{(i)}: \R \rightarrow \R$ for each layer $i$. Figure~\ref{tab:pt_vis} plots, for a fixed latent value, the predicted mask score as a function of the initialization value (low mask scores are pruned, while high mask scores are not). These plots suggest that, in the MLP setting, neural functionals are learning something similar to magnitude pruning of the initialization. In our setting, this turns out to be a strong baseline for lottery ticket mask prediction: the test accuracy of models pruned with the modified network is 95.0\%.

\begin{figure}
    \centering
    \includegraphics[width=\textwidth]{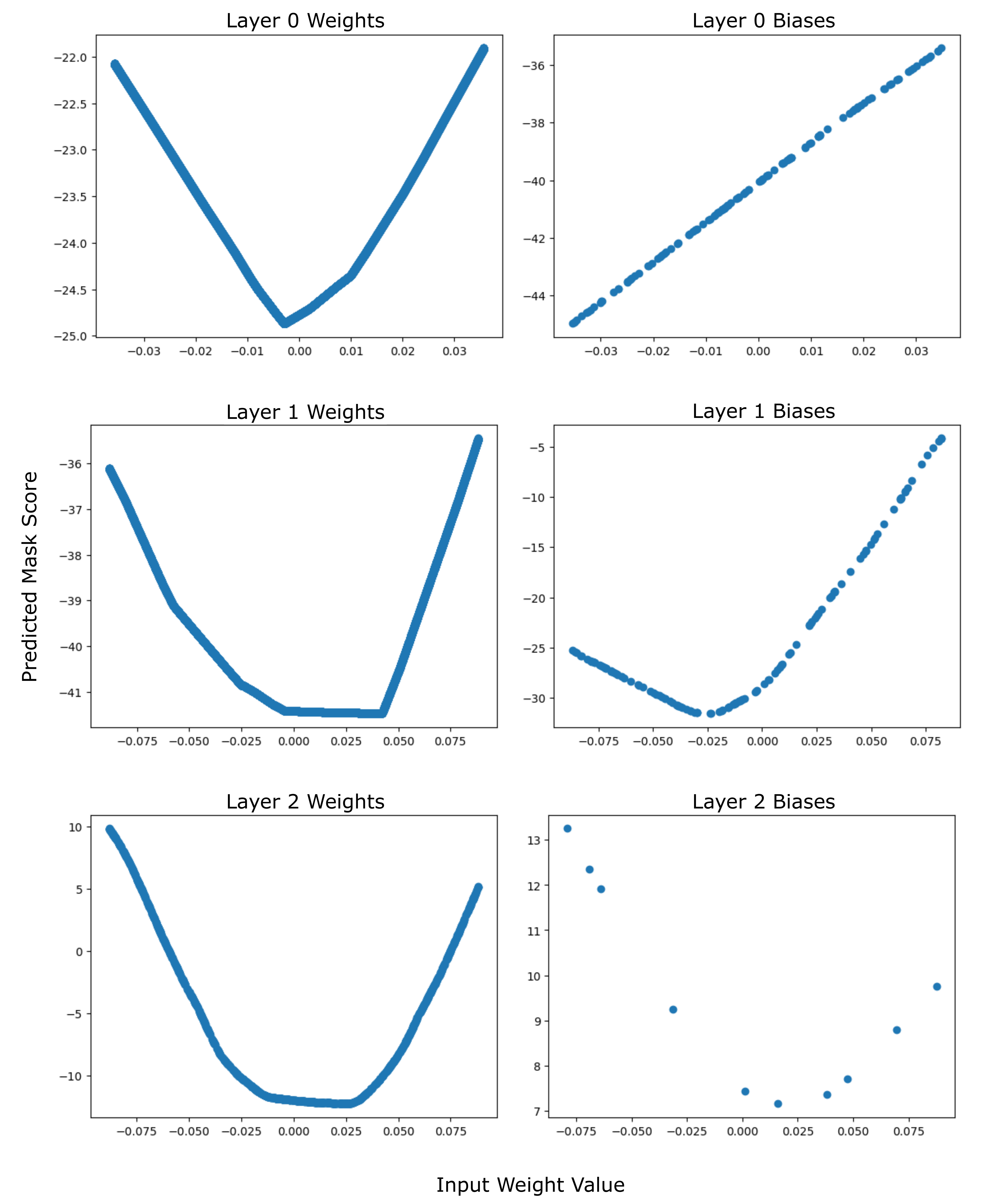}
    \caption{Mask scores vs weight magnitude for modified \ptours. }
    \label{tab:pt_vis}
\end{figure}

\subsection{Predicting MLP generalization from weights}
\begin{table}[]
    \centering
    \caption{Kendall's $\tau$ coefficient and $R^2$ between predicted and actual test accuracies of three- and five-layer MLPs trained on MNIST. Our equivariant neural functionals outperform the baseline from \citep{unterthiner2020predicting} which predicts generalization using only simple weight statistics as features. Uncertainties indicate standard error over five runs.}
    \begin{tabular}{ccrrr}
      \toprule
       & & \fullours & \ours & 
\statnet \\
      \midrule
      \multirow{2}{*}{\centering$\tau$} & 3-Layer & $\mathbf{0.876\pm 0.003}$ & $0.859 \pm 0.002$ & $0.854\pm 0.002$  \\
        & 5-Layer & $\mathbf{0.871\pm 0.001}$ & $0.855\pm0.001$ & $0.860\pm0.001$ \\
        \midrule
      \multirow{2}{*}{\centering$R^2$} & 3-Layer & $\mathbf{0.957 \pm 0.003}$ & $0.9424\pm 0.003$ & $0.937 \pm 0.002$ \\
        & 5-Layer & $\mathbf{0.956 \pm 0.002}$ & $0.947 \pm 0.001$ & $0.950 \pm 0.001$ \\
      \bottomrule
    \end{tabular}
    \label{tab:mlp}
\end{table}
\label{app:add_exp}
In addition to predicting generalization on the Small CNN Zoo benchmark (Section \ref{sec:pred_cnn_gen}), we also construct our own datasets to evaluate predicting generalization on MLPs. Specifically, we study three- and five-layer MLPs with 128 units in each hidden layer. For each of the two architectures, we train 2,000 MLPs on MNIST with varying optimization hyperparameters, and save 10 randomly-selected checkpoints from each run to construct a dataset of 20,000 (weight, test accuracy) pairs. Runs are partitioned according to a $90\%$ / $10\%$ split for training and testing.\\

We evaluate \fullours and \ours on this task and compare them to the \statnet baseline \citep{unterthiner2020predicting} which predicts test accuracy from hand-crafted features extracted from the weights. Table \ref{tab:mlp} shows that \ours and \statnet are broadly comparable, while \fullours consistently outperform other methods across both datasets in two measures of correlation: Kendall's tau and $R^2$. These results confirm that processing the raw weights with permutation equivariant neural functionals can lead to greater predictive power when assessing generalization from weights.

\end{document}